\documentclass{article}
\usepackage[normalem]{ulem}

\usepackage{PRIMEarxiv}

\usepackage[utf8]{inputenc}
\usepackage[T1]{fontenc}

\usepackage{hyperref}

\usepackage{url}
\usepackage{booktabs}
\usepackage{nicefrac}
\usepackage{microtype}
\usepackage{lipsum}
\usepackage{fancyhdr}
\usepackage{graphicx}
\usepackage{subcaption}
\usepackage{float}
\usepackage{multirow}
\usepackage{authblk}
\usepackage{xcolor}
\usepackage{enumitem}

\usepackage{amsmath}
\usepackage{amsfonts}
\usepackage{amssymb}
\usepackage{mathtools}
\usepackage{amsthm}

\usepackage{algorithm}
\usepackage{algpseudocode}

\usepackage{tikz}
\usetikzlibrary{positioning, shapes.geometric, arrows.meta}

\theoremstyle{plain}
\newtheorem{theorem}{Theorem}
\newtheorem{lemma}[theorem]{Lemma}
\newtheorem{definition}[theorem]{Definition}
\newtheorem{assumption}[theorem]{Assumption}
\newtheorem{proposition}{Proposition}

\usepackage[numbers,sort&compress]{natbib} %

\usepackage{xcolor}
\usepackage[normalem]{ulem}

\newif\ifshowedits
\showeditstrue 

\ifshowedits
  \newcommand{\oldtext}[1]{{\color{gray}\sout{#1}}}
  
\else
  \newcommand{\oldtext}[1]{}

\fi

\graphicspath{{media/}}

\newcounter{pr}

\title{Representation Learning via Composite Subspace Flow Matching}

\author[1]{\textbf{José Afonso}\thanks{Corresponding author: josefafonso@tecnico.ulisboa.pt}\hspace{0.9ex}}
\author[2]{\textbf{Mário Figueiredo}}
\affil[1]{Instituto Superior Técnico, Lisboa, Portugal}
\affil[2]{Instituto de Telecomunicações, Instituto Superior Técnico, Lisboa, Portugal}

\begin{document}
\maketitle

\begin{abstract}
\noindent Distilling semantic meaning from raw image data remains a challenge in machine learning. 
Current self-supervised paradigms face a trade-off: joint-embedding architectures produce relevant features by aggressively discarding intra-class variance, irreversibly spatial and reconstruction. 
Conversely, generative models preserve spatial details but struggle to isolate semantics, as their optimization is dominated by pixel variations.
To bridge this divide, we propose composite subspace flow matching (CSFM). 
Rather than corrupting data with unstructured isotropic noise, CSFM formulates representation learning as the variance-minimizing inversion of a deterministic degradation curriculum across spectral, color, and illumination subspaces. 
By bounding trajectory variance with hard temporal style truncation, a capacity-constrained encoder is driven to ignore transient style artifacts and encode persistent content. 
Empirical evaluations across CIFAR-10, CIFAR-100, and STL-10 show that CSFM yields competitive discriminative transferability while preserving generative synthesis across {both} convolutional and vision transformer architectures.
By avoiding the collapse pathologies of non-generative methods, CSFM forms a smooth latent space that outperforms standard optimal transport flows on linear probing and few-shot classification, while enabling coherent synthesis and trajectory interpolation.
These results confirm that steering the physical dynamics of generative flows offers a scalable paradigm for learning representations that are simultaneously linearly separable and generatively invertible.
\end{abstract}



\section{Introduction}
\label{sec:intro}

The primary objective of self-supervised representation learning is to extract informative, high-level, and useful features for downstream tasks from unlabeled raw data \cite{bengio2014representationlearningreviewnew, LeCun2015, Le_Khac_2020}. 
While self-supervised learning (SSL) has catalyzed advances across computer vision, natural language processing, and world models for autonomous decision-making \cite{he2015deepresiduallearningimage, devlin2019bertpretrainingdeepbidirectional, hafner2024masteringdiversedomainsworld, lecun2022path}, extracting representations from unlabelled visual data remains challenging due to the high noise-to-signal ratio inherent to pixel space \cite{vanassel2025jointembeddingvsreconstruction}. 
Unlike discrete textual tokens, which represent semantically compressed units, raw pixels are uncompressed physical measurements \cite{vanassel2025jointembeddingvsreconstruction}. 
Consequently, high-magnitude nuisance variables, such as localized textures, illumination shifts, and background clutter, statistically dominate the input variance, frequently overshadowing the underlying semantic targets \cite{balestriero2023cookbook, balestriero2024learningreconstructionproducesuninformative}.

To address these nuisance factors, visual SSL research has largely split into two paradigms with conflicting trade-offs \cite{vanassel2025jointembeddingvsreconstruction}. On the one hand, reconstruction-based architectures, such as masked autoencoders \cite{he2021maskedautoencodersscalablevision} and denoising autoencoders \cite{pascal2008dae, bengio2013generalizeddenoisingautoencodersgenerative}, model the data distribution by inverting input-space corruptions. However, because standard $L_2$ pixel recovery aligns with raw variance rather than semantic abstraction, these models remain biased toward unpredictable, high-variance details \cite{zhang2018unreasonableeffectivenessdeepfeatures, lecun2022path, balestriero2024learningreconstructionproducesuninformative}. Conversely, joint-embedding architectures (JEAs), such as VICReg \cite{bardes2022vicregvarianceinvariancecovarianceregularizationselfsupervised}, and LeJEPA \cite{balestriero2025lejepaprovablescalableselfsupervised}, circumvent pixel synthesis entirely by enforcing invariance across augmented views in a latent representation space \cite{grill2020bootstraplatentnewapproach, caron2021emergingpropertiesselfsupervisedvision}. Although this abstraction shields JEAs from nuisance variations, it leaves them vulnerable to informational and dimensional collapse \cite{zbontar2021barlowtwinsselfsupervisedlearning, jing2022understandingdimensionalcollapsecontrastive}.

More fundamentally, recent theoretical work reveals a structural friction between discriminative extraction and generative reversibility \cite{bizeul2024probabilisticmodelselfsupervisedlearning, vanassel2025jointembeddingvsreconstruction, balestriero2024learningreconstructionproducesuninformative}. 
To maximize linear separability on semantic tasks, JEAs suppress all intra-class sensitivity. 
While this invariance maximizes classification margins, it permanently deletes high-frequency and spatial information, rendering the learned representations structurally irreversible and incapable of image synthesis \cite{bizeul2024probabilisticmodelselfsupervisedlearning, vonkügelgen2022selfsupervisedlearningdataaugmentations}.

Continuous-time generative models, specifically continuous normalizing flows and diffusion models \cite{ho2020denoisingdiffusionprobabilisticmodels, lipman2023flowmatchinggenerativemodeling, song2020score}, provide fully invertible spatial transformations that preserve geometry \cite{chen2018neural, lipman2023flowmatchinggenerativemodeling}. 
Recognizing their capacity to model complex distributions without spatial loss, a recent paradigm shift repurposes continuous generative trajectories as proxy objectives for representation learning \cite{preechakul2022diffusionautoencodersmeaningfuldecodable, mittal23DRL, ma2025learningstraightflowsvariational}. 
By conditioning a continuous vector field on a bottlenecked latent code, the encoder is driven to distill the sufficient statistics required to resolve trajectory ambiguities \cite{ma2025learningstraightflowsvariational, hudson2023soda}. 
Because this objective relies on continuous trajectory reconstruction rather than explicit embedding-space invariance, it avoids the dimensional collapse common in JEAs while maintaining spatial reversibility \cite{li2024return}.
 
Despite this promise, continuous representation models face a significant limitation: standard frameworks rely almost exclusively on isotropic Gaussian noise \cite{song2020score, ho2020denoisingdiffusionprobabilisticmodels, ma2024sitexploringflowdiffusionbased}. This spatially uncorrelated corruption perturbs all image coordinates independently and identically, without accounting for the underlying structure of visual data.
Recent generative works have expanded beyond pure Gaussian noise by designing physical degradation trajectories, such as blurring or masking \cite{bansal2022colddiffusioninvertingarbitrary, daras2022softdiffusionscorematching, hoogeboom2024blurringdiffusionmodels}. 
However, adapting these cumulative multi-operator processes to representation learning proves challenging: sequential noise injections incur an intractable covariance inversion cost and introduce an irreducible target noise floor that disrupts encoder optimization.

To resolve these limitations while balancing representation quality with generative reversibility, we propose \textit{composite subspace flow matching} (CSFM). 
It replaces uniform corruption with an ordered sequence of deterministic subspace projections spanning the frequency, color, and illumination domains. 
By pairing this composite drift with a single, decoupled isotropic noise schedule, CSFM yields a simulation-free target vector field that nullifies irreducible target noise with minimal computational overhead.
Under this continuous curriculum, content is defined by temporal persistence under physical transport. 
CSFM isolates content within a capacity-bounded latent code while offloading transient spatial details to the generative vector field, balancing downstream linear separability with a smooth, reversible generative manifold.

The core contributions of this work are three-fold:
\begin{enumerate}[leftmargin=*]
    \item {Composite subspace flow matching (CSFM):} we introduce a simulation-free flow matching framework that couples an ordered visual degradation path with a decoupled isotropic noise schedule, eliminating target velocity variance with minimal memory overhead
    \item {Optimal encoder geometry:} we show that minimizing trajectory variance induces a block-diagonal precision geometry, {prevents features from the style subspace from dominating encoder capacity}, and proves that style truncation is necessary for content alignment.
    \item {Comprehensive empirical validation:} across CIFAR-10, CIFAR-100, and STL-10 ($96 \times 96$), using CNN and vision transformer backbones, CSFM outperforms unguided flows and autoencoders in linear and few-shot probing while preserving an invertible, connected latent space for generative decoding.
\end{enumerate}

The remainder of this paper is organized as follows. 
Section~\ref{sec:background} reviews self-supervised representation learning paradigms, continuous normalizing flows, trajectory conditioning, and sketched Gaussian regularization.
Section~\ref{sec:methodology} introduces CSFM, including the composite projection curriculum, asymmetric gradient routing, and forward-mode analytical computation. 
Section~\ref{sec:theory} provides theoretical proofs for subspace disentanglement, capacity saturation, and the style truncation requirement. 
Section~\ref{sec:experiments} presents empirical evaluations on representation quality, generative fidelity, and scalability across architectures and resolutions. 
Section~\ref{sec:conclusion} presents our conclusions and outlines directions for future work.

\section{Background and Related Work}
\label{sec:background}

\subsection{Paradigms of Representation Learning: JEAs vs. Generative Models}
\label{sec:bg_paradigms}

The goal of self-supervised representation learning is to extract from raw sensory data useful features for downstream tasks, while discarding task-irrelevant style and nuisance \cite{achille2018information, bardes2022vicregvarianceinvariancecovarianceregularizationselfsupervised, balestriero2025lejepaprovablescalableselfsupervised}. 
Following recent theoretical formulations \cite{vonkügelgen2022selfsupervisedlearningdataaugmentations, vanassel2025jointembeddingvsreconstruction}, an observation $x \in \mathbb{R}^D$ is decomposed into complementary orthogonal subspaces $\mathbb{R}^D = \mathcal{X}_c \oplus \mathcal{X}_s$ {via}
\begin{equation}
    x = x_c + x_s,
    \label{eq:additive_subspace}
\end{equation}
where content $x_c \in \mathcal{X}_c$ lies in a low-dimensional subspace ($\dim(\mathcal{X}_c) = c \ll D$) and style $x_s \in \mathcal{X}_s$ occupies the orthogonal complement ($\mathcal{X}_s = \mathcal{X}_c^\perp$) \cite{vanassel2025jointembeddingvsreconstruction}.

Existing self-supervised methodologies target data decomposition through two diverging paradigms with contrasting trade-offs \cite{vanassel2025jointembeddingvsreconstruction}. 
Reconstruction-based generative architectures, such as denoising autoencoders (DAE) \cite{pascal2008dae, bengio2013generalizeddenoisingautoencodersgenerative} and masked autoencoders (MAE) \cite{he2021maskedautoencodersscalablevision}, learn by inverting input corruptions in pixel space \cite{vanassel2025jointembeddingvsreconstruction}. 
While this ensures generative reversibility and preserves fine spatial, standard $L_2$ pixel recovery optimizes for total raw variance rather than semantic features \cite{zhang2018unreasonableeffectivenessdeepfeatures, lecun2022path}. 
Consequently, the encoder exhausts substantial capacity fitting textures and high-variance background noise \cite{vanassel2025jointembeddingvsreconstruction}. 
Conversely, JEAs bypass pixel synthesis entirely by training an encoder to match latent representations across independently augmented views \cite{chen2020simpleframeworkcontrastivelearning, grill2020bootstraplatentnewapproach, caron2021emergingpropertiesselfsupervisedvision}. 
Predicting strictly in latent space protects JEAs from raw pixel \cite{vanassel2025jointembeddingvsreconstruction}. 
However, by aggressively discarding all variance orthogonal to the hand-crafted invariances, JEAs permanently {discard spatial structure}, yielding {structurally irreversible representations that are insufficient for image reconstruction} \cite{bizeul2024probabilisticmodelselfsupervisedlearning, vonkügelgen2022selfsupervisedlearningdataaugmentations}.

\paragraph{Representational Collapse and Latent Regularization.}
A critical failure mode of optimizing latent invariance in JEAs without generative constraints is representational collapse, wherein representations degrade to a trivial constant vector (information collapse) or span a degenerate lower-dimensional subspace (dimensional collapse) \cite{zbontar2021barlowtwinsselfsupervisedlearning, jing2022understandingdimensionalcollapsecontrastive}. 
While contrastive objectives \cite{chen2020simpleframeworkcontrastivelearning} and covariance decorrelation penalties (e.g. VICReg \cite{bardes2022vicregvarianceinvariancecovarianceregularizationselfsupervised}, Barlow Twins \cite{zbontar2021barlowtwinsselfsupervisedlearning}) prevent collapse, second-moment decorrelation {incurs quadratic computational complexity and is insufficient to guarantee true Gaussianity, which is provably optimal for minimizing downstream prediction risk \cite{balestriero2025lejepaprovablescalableselfsupervised}.}
Leveraging the Cram\'er-Wold theorem, SIGReg matches random $1$D projections of the empirical characteristic function against an analytical standard normal with linear complexity, providing a robust, collapse-free bottleneck without variational sampling noise \cite{balestriero2025lejepaprovablescalableselfsupervised}.

\subsection{Continuous Normalizing Flows and Conditional Flow Matching}
\label{sec:bg_cfm}

Continuous normalizing flows (CNFs) generate data along a continuous probability path $p_t(x_t)$ ($t \in [0, 1]$) that transports a standard Gaussian prior $p_{\text{prior}} \equiv \mathcal{N}(0, I_D)$ at $t = 1$ to the clean data distribution $p_{\text{data}}$ at $t = 0$ \footnote{Canonical flow matching \cite{lipman2023flowmatchinggenerativemodeling, albergo2023buildingnormalizingflowsstochastic} sets prior noise at $t=0$ and clean data at $t=1$. We invert this convention ($t=0$ for clean data, $t=1$ for noise) to frame representation learning as extracting features from uncorrupted data $x_0$ while degrading images forward in time ($t > 0$).}  \cite{chen2018neural, lipman2023flowmatchinggenerativemodeling}.
Conservation of probability mass dictates that this density path is governed by a time-dependent vector field $v_t: \mathbb{R}^D \to \mathbb{R}^D$ satisfying the continuity equation \cite{lipman2023flowmatchinggenerativemodeling, albergo2025stochasticinterpolantsunifyingframework, tong2024improvinggeneralizingflowbasedgenerative}:
\begin{equation}
    \partial_t p_t(x_t) + \nabla \cdot \big( p_t(x_t) v_t(x_t) \big) = 0.
    \label{eq:continuity}
\end{equation}
Sampling from $p_{\text{data}}$ is achieved by drawing $x_1 \sim \mathcal{N}(0, I_D)$ and integrating the characteristic ordinary differential equation (ODE) $\frac{d}{dt} x_t = v_t(x_t)$ backward from $t = 1$ to $t = 0$ \cite{lipman2023flowmatchinggenerativemodeling, albergo2025stochasticinterpolantsunifyingframework, tong2024improvinggeneralizingflowbasedgenerative}.

Because computing the marginal velocity field $v_t(x_t) = \mathbb{E}_{x_0, x_1}[u_t(x_t \mid x_0, x_1) \mid x_t]$ requires intractable dataset-wide integration, conditional flow matching (CFM) introduces the independent coupling $\pi(x_0, x_1) = p_{\text{data}}(x_0) p_{\text{prior}}(x_1)$ and regresses a neural vector field $v_\theta(x_t, t)$ against an analytical conditional target $u_t(x_t \mid x_0, x_1)$ \cite{lipman2023flowmatchinggenerativemodeling, albergo2023buildingnormalizingflowsstochastic}
\begin{equation}
    \mathcal{L}_{\text{CFM}}(\theta) = \mathbb{E}_{t, x_0, x_1} \left[ \| u_t(x_t \mid x_0, x_1) - v_\theta(x_t, t) \|_2^2 \right],
    \label{eq:cfm_loss}
\end{equation}
where $t \sim \mathcal{U}(0, 1)$ and $x_t \sim p_t(x_t \mid x_0, x_1)$. 
Provided $p_t(x_t) > 0$, the gradients of $\mathcal{L}_{\text{CFM}}(\theta)$ are identical to those of the intractable marginal regression objective \cite{lipman2023flowmatchinggenerativemodeling, albergo2025stochasticinterpolantsunifyingframework, tong2024improvinggeneralizingflowbasedgenerative}. 
Under optimal transport CFM (OT-CFM), trajectories follow straight lines: $x_t = (1 - t)x_0 + t x_1$, yielding the constant target velocity $u_t(x_t \mid x_0, x_1) = x_1 - x_0$ \cite{lipman2023flowmatchinggenerativemodeling, tong2024improvinggeneralizingflowbasedgenerative}.

\paragraph{The Irreducible Target Variance Floor.} 
Under an $L_2$ regression penalty, the optimal vector field converges to the conditional expectation $v_\theta^*(x_t, t) = \mathbb{E}[u_t \mid x_t]$ \cite{bishop2006pattern}. 
Decomposing the objective around $v_\theta^*$ reveals a lower bound \cite{ma2025learningstraightflowsvariational}:
\begin{equation}
    \mathcal{L}_{\text{CFM}}(\theta^*) \geq \mathbb{E}_{t, x_t} \left[ \operatorname{Tr}\Big(\operatorname{Var}(u_t \mid x_t)\Big) \right].
    \label{eq:irreducible_variance_bg}
\end{equation}
The loss is bounded below by the \textit{irreducible target variance} arising from trajectory intersections where distinct conditional paths cross at identical space-time coordinates $x_t$ \cite{ma2025learningstraightflowsvariational, mittal23DRL}. 
{Eliminating this residual target variance} requires conditioning the flow field on a latent representation $z$ {to disambiguate crossing trajectories}, which forms the theoretical foundation for continuous flow-based representation learning \cite{mittal23DRL, ma2025learningstraightflowsvariational}.

\paragraph{Structured Degradation vs. Decoupled Trajectories.}
Standard continuous generative models rely on isotropic Gaussian noise, which perturbs all pixel coordinates independently and identically, {failing to capture} the structured spatial, spectral, and photometric correlations of natural images \cite{bansal2022colddiffusioninvertingarbitrary}.
While generalized diffusion models incorporate physical corruptions like blurring or heat dissipation \cite{daras2022softdiffusionscorematching, hoogeboom2024blurringdiffusionmodels, rissanen2023generative}, they remain restricted to single diagonalizable operators to preserve analytical tractability. 
{Extending this formulation to multiple operators reveals a computational divide: when the corruption operators commute, they admit a simultaneous eigenbasis, enabling closed-form conditioning and evaluation in $\mathcal{O}(D \log D)$ time. 
Conversely, composing non-commuting operators destroys this joint diagonalization, producing a dense marginal covariance whose inversion incurs an intractable $\mathcal{O}(D^3)$ computational cost. 
Regardless of commutativity, injecting independent noise across paths produces intersecting trajectories, creating an irreducible target variance floor that cannot be eliminated without latent conditioning.} 
In parallel, masked modeling reveals that corrupting global spectral bases rather than localized spatial patches enhances semantic representation learning \cite{bizeul2025pixelscomponentseigenvectormasking}.
In a similar spirit of leveraging structured data transformations, CSFM embeds structured physical degradation into a continuous generative model: pairing a deterministic drift across canonical visual subspaces with a decoupled isotropic noise path steers the encoder toward persistent content at minimal computational cost.

\subsection{Flow-Based Representation Learning and Deterministic Conditioning}
\label{sec:bg_rfm}

To resolve trajectory intersections and repurpose continuous models for representation learning, we leverage the framework of variational flow matching (VFM) \cite{ma2025learningstraightflowsvariational, guo2025variationalrectifiedflowmatching}, which conditions the vector field on {an encoder representation $z = E_\phi(x_0) \in \mathbb{R}^b$ ($b \ll D$): $v_\theta(x_t, t, z)$. 
The output dimension $b$ defines the architectural bottleneck of the encoder $E_\phi$.} While standard VFM formulations employ a stochastic posterior $q_\phi(z \mid x_0) = \mathcal{N}(\mu_\phi, \Sigma_\phi)$ to facilitate random generation \cite{kingma2022autoencodingvariationalbayes}, sampling noise inflates intra-class variance and degrades downstream linear separability \cite{ghosh2020variationaldeterministicautoencoders, chen2020simpleframeworkcontrastivelearning}.
To {eliminate this sampling variance}, we collapse the posterior into a Dirac distribution centered on an encoder $E_\phi$: $q_\phi(z \mid x_0) = \delta(z - E_\phi(x_0))$ \cite{ghosh2020variationaldeterministicautoencoders}. 
The objective simplifies to joint optimization over the deterministic encoder $\phi$ and vector field $\theta$:
\begin{equation}
    \mathcal{L}_{\text{CFM}}(\theta, \phi) = \mathbb{E}_{t, x_0, x_1} \left[ \| u_t(x_t \mid x_0, x_1) - v_\theta(x_t, t, E_\phi(x_0)) \|_2^2 \right].
    \label{eq:joint_cfm}
\end{equation}
Because intermediate states $x_t$ depend solely on boundary pairs $(x_0, x_1)$, marginalizing over $z = E_\phi(x_0)$ preserves the true marginal probability path $p_t(x_t)$ and maintains generative invertibility \cite{ma2025learningstraightflowsvariational, guo2025variationalrectifiedflowmatching}.

\paragraph{Deterministic Conditioning as a Variance Bottleneck.} 
Conditioning on $z = E_\phi(x_0)$ alters the irreducible variance floor in Eq.~\eqref{eq:irreducible_variance_bg}. 
By the law of total variance:
\begin{equation}
    \operatorname{Var}(u_t \mid x_t) = \mathbb{E}_{z \mid x_t} \left[ \operatorname{Var}(u_t \mid x_t, z) \right] + \operatorname{Var}_{z \mid x_t} \left( \mathbb{E}[u_t \mid x_t, z] \right).
    \label{eq:total_variance}
\end{equation}
Because the covariance of conditional expectations is positive semi-definite, conditioning on the representation reduces the residual velocity variance:
\begin{equation}
    \mathbb{E}_{t, x_t} \left[ \operatorname{Tr}\big(\operatorname{Var}(u_t \mid x_t)\big) \right] \ge \mathbb{E}_{t, x_t, z} \left[ \operatorname{Tr}\big(\operatorname{Var}(u_t \mid x_t, E_\phi(x_0))\big) \right].
    \label{eq:variance_reduction_bound}
\end{equation}
Under the capacity bottleneck $b \ll D$, the encoder cannot memorize $x_0$. 
{By minimizing Eq.~\eqref{eq:joint_cfm}, the encoder $E_\phi$ learns} minimal sufficient statistics of $x_0$ that maximize variance reduction along the continuous trajectories \cite{mittal23DRL, ma2025learningstraightflowsvariational}.

\paragraph{Dynamics-Driven Persistence vs. Augmentation Invariance.} 
This deterministic formulation shifts how content and style are separated. 
In joint-embedding architectures, content is defined by invariance to hand-crafted data augmentations (e.g. random crops and color jitter) \cite{tian2020makesgoodviewscontrastive, eastwood2024selfsuperviseddisentanglementleveragingstructure}, discarding perturbed features as style and risking shortcut learning if heuristics misalign with underlying semantics \cite{vanassel2025jointembeddingvsreconstruction, balestriero2024learningreconstructionproducesuninformative}. 
In continuous flows, content is defined by \textit{temporal persistence} under physical transport. 
Because transient fine-scale variations vanish early along the degradation trajectory, they induce large conditional velocity variance if predicted from degraded states. 
Consequently, the capacity-bounded encoder is driven to distill features that resolve global trajectory ambiguity, isolating content from continuous transport dynamics rather than synthetic view corruptions.

\section{Methodology: Composite Subspace Flow Matching}
\label{sec:methodology}

\subsection{Composite Orthogonal Subspaces}
\label{sec:composite_subspaces}

{To formalize the separation of content from style}, we partition the $D$-dimensional observation space $\mathbb{R}^D$ into complementary content ($\mathcal{X}_c$) and style ($\mathcal{X}_s$) subspaces. 
This decomposition is governed by a pair of self-adjoint orthogonal projection operators $P_c, P_s \in \mathbb{R}^{D \times D}$ satisfying:
{
\begin{equation}
    P_c + P_s = I_D, \quad P_c P_s = \mathbf{0}, \quad P_c^2 = P_c = P_c^T, \quad P_s^2 = P_s = P_s^T.
    \label{eq:projector_axioms}
\end{equation}
}
Eq.~\eqref{eq:projector_axioms} guarantees a lossless orthogonal decomposition of any sample: $x_0 = x_{0,c} + x_{0,s}$, where $x_{0,c} = P_c x_0 \in \mathcal{X}_c$ isolates the content component and $x_{0,s} = P_s x_0 \in \mathcal{X}_s$ denotes the style component.
We instantiate these operators deterministically across three physical domains that mirror visual perception:

\paragraph{Frequency Domain (DCT-II Basis).} 
Macroscopic visual topology resides predominantly in low spatial frequencies, whereas fine-grained textures occupy high frequencies \cite{geirhos2022imagenettrainedcnnsbiasedtexture, bansal2022colddiffusioninvertingarbitrary}. 
We project images into the real-valued frequency domain using the 2D orthonormal discrete cosine transform (DCT-II), which achieves strong energy compaction without complex arithmetic or periodic boundary artifacts \cite{ahmed1974discrete}.
For an image of dimensions $H \times W$, the separable transform evaluates as $Y = C_H X C_W^T$, where $C_N \in \mathbb{R}^{N \times N}$ is the orthogonal DCT-II matrix with entries $C_{k, n} = \sqrt{\frac{2}{N}} \alpha_k \cos\big( \frac{\pi (2n+1)k}{2N} \big)$ ($\alpha_0 = 1/\sqrt{2}$ and $\alpha_k = 1$ for $k > 0$). 
Let $d(u, v) = \frac{1}{\sqrt{2}} \sqrt{(u/W)^2 + (v/H)^2} \in [0, 1]$ denote normalized radial frequency coordinates. 
While an ideal sharp low-pass filter satisfies exact projection idempotence ($P_c^2 = P_c$), sharp frequency truncation acts as a spatial $\operatorname{sinc}$ convolution, inducing Gibbs ringing artifacts that bias neural regressors toward artificial ripple textures \cite{oppenheim2010discrete}. 
To ensure spatial stability, we implement the spectral content projector as a smooth radial Gaussian filter:
\begin{equation}
    P_c^{\text{freq}}(u, v) = \exp\left( - \frac{d(u, v)^2}{2 \sigma_{\text{low}}^2} \right), \quad P_s^{\text{freq}} = I_D - P_c^{\text{freq}},
    \label{eq:proj_freq}
\end{equation}
where $\sigma_{\text{low}} = r_{\text{low}} / W$ sets the radial cutoff. 
Because $P_c^{\text{freq}} \in (0, 1]$, the operator relaxes strict idempotence with a bounded cross-term $P_c^{\text{freq}} P_s^{\text{freq}} \le 0.25$ that is confined to a narrow transition band around $\sigma_{\text{low}}$ and vanishes asymptotically elsewhere, preserving non-negative impulse responses without structural ringing.

\paragraph{Color Domain (Luminance--Chrominance Basis).}
To prevent the encoder from exploiting superficial color distributions as predictive shortcuts \cite{chen2020simpleframeworkcontrastivelearning}, we decouple structural luminance from chromatic variations. 
Let $U_{\rm color} \in \mathbb{R}^{3 \times 3}$ denote the orthonormal Ohta color transformation matrix ($U_{\rm color}^T U_{\rm color} = I_3$) \cite{ohta1980color}:
\begin{equation}
    U_{\rm color} = \begin{pmatrix} 
    1/\sqrt{3} & 1/\sqrt{3} & 1/\sqrt{3} \\ 
    1/\sqrt{2} & -1/\sqrt{2} & 0 \\ 
    1/\sqrt{6} & 1/\sqrt{6} & -2/\sqrt{6} 
    \end{pmatrix},
    \label{eq:color_matrix}
\end{equation}
where the first row isolates achromatic luminance ($Y$), and the remaining orthogonal rows span the chromatic subspace. 
For an image with $N = HW$ pixels ($D = 3N$), the orthogonal content and style projectors are defined via Kronecker products:
\begin{align}
    P_c^{\text{color}} &= \big(U_{\rm color}^T \operatorname{diag}(1, 0, 0) U_{\rm color}\big) \otimes I_N, \\
    P_s^{\text{color}} &= \big(U_{\rm color}^T \operatorname{diag}(0, 1, 1) U_{\rm color}\big) \otimes I_N.
    \label{eq:proj_color}
\end{align}

\paragraph{Illumination Domain (DC Spatial Basis).} 
To decouple intrinsic scene contrast from global ambient lighting shifts \cite{ulyanov2017instancenormalizationmissingingredient}, we isolate the uniform spatial DC component. 
Let $v_{\text{DC}} = \frac{1}{\sqrt{D}}\mathbf{1} \in \mathbb{R}^D$ denote the normalized all-ones vector. The projectors evaluate to:
\begin{equation}
    P_c^{\text{illum}} = I_D - v_{\text{DC}} v_{\text{DC}}^T, \quad P_s^{\text{illum}} = v_{\text{DC}} v_{\text{DC}}^T.
    \label{eq:proj_illum}
\end{equation}

{Crucially, these operators mutually commute and are simultaneously diagonalized by the separable orthonormal basis $Q = U_{\rm color} \otimes \mathcal{C} \in \mathbb{R}^{D \times D}$, with 2D DCT-II basis $\mathcal{C} = C_W \otimes C_H$. 
This joint diagonal form bypasses dense matrix operations, allowing any composite projection or inverse covariance filtering to be evaluated in $\mathcal{O}(D \log D)$ time via the 2D fast cosine transform (proof in Appendix~\ref{app:simultaneous_diagonalization}).}

\subsection{Decoupled Flow Matching and Variance Nullification}
\label{sec:decoupled_path}

Conditioning the vector field regressor on a deterministic latent representation $z = E_\phi(x_0) \in \mathbb{R}^b$ yields the joint objective in Eq.~\eqref{eq:joint_cfm}. 
Because the optimal vector field is the conditional expectation $v^*(x_t, t, z) = \mathbb{E}[u_t \mid x_t, z]$, the law of total variance separates the expected regression loss into three terms (derived in Appendix~\ref{app:tripartite_derivation}):
\begin{equation}
    \mathcal{L}_{\text{CFM}}(\theta, \phi) = \underbrace{\mathbb{E} \| v_\theta - v^* \|_2^2}_{\mathcal{E}_{\text{model}}(\theta)} + 
    \underbrace{\mathbb{E}_{t, x_t, z} \left[ \operatorname{Tr}\Big(\operatorname{Var}_{x_0 \mid x_t, z}\big(\mathbb{E}[u_t \mid x_t, x_0]\big)\Big) \right]}_{\mathcal{E}_{\text{repr}}(\phi)} + 
    \underbrace{\mathbb{E}_{t, x_t, x_0} \left[ \operatorname{Tr}\big(\operatorname{Var}(u_t \mid x_t, x_0)\big) \right]}_{\mathcal{E}_{\text{target}}}.
    \label{eq:tripartite_loss}
\end{equation}
The term $\mathcal{E}_{\text{model}}(\theta)$ captures the reducible approximation error of $v_\theta$ relative to the Bayes-optimal field $v^*$. Minimizing $\mathcal{E}_{\text{repr}}(\phi)$ drives representation learning by reducing posterior uncertainty over $x_0$ given $(x_t, z)$, which forces $E_\phi$ to extract the sufficient statistics required to guide the velocity field. Lastly, $\mathcal{E}_{\text{target}}$ represents the irreducible variance of $u_t$ given $(x_t, x_0)$, setting an unavoidable, parameter-independent noise floor.

\paragraph{Failure of Cumulative Stochastic Degradation.}
A naive approach inspired by generalized diffusion \cite{hoogeboom2024blurringdiffusionmodels, daras2022softdiffusionscorematching} constructs multi-stage degradation across $K$ physical domains via sequential operator cascades with independent noise injections: $x_t = M_t x_0 + \mathbf{B}_t \boldsymbol{\epsilon}$, where $M_t \in \mathbb{R}^{D \times D}$ is the deterministic linear operator cascade, $\boldsymbol{\epsilon} = [\epsilon_1^T, \dots, \epsilon_K^T]^T \in \mathbb{R}^{KD}$ concatenates independent Gaussian noise vectors, and $\mathbf{B}_t \in \mathbb{R}^{D \times KD}$ is the composite noise schedule. 
Because $KD > D$, this cumulative formulation forms an underdetermined linear system for any given observation $x_t$. 
As detailed in Appendix~\ref{app:cumulative_limitations}, conditioning the target velocity requires inverting the marginal covariance $\Sigma_t = \mathbf{B}_t \mathbf{B}_t^T \in \mathbb{R}^{D \times D}${, which incurs an intractable $\mathcal{O}(D^3)$ computational cost whenever operators fail to commute, and leaves an irreducible target variance floor $\mathcal{E}_{\text{target}} > 0$ along the null space of $\mathbf{B}_t$ that drowns backpropagated representation gradients.}

\paragraph{Decoupled Trajectory and Exact Variance Nullification.}
We eliminate this noise floor by defining a \textit{decoupled} trajectory that combines a composite deterministic drift operator $M_t$ with a single isotropic noise path $B_t = t I_D$:
\begin{equation}
    x_t = M_t x_0 + t x_1 \implies u_t = \dot{M}_t x_0 + x_1.
    \label{eq:decoupled_path}
\end{equation}
Because $B_t = t I_D$ is invertible for all $t \in (0, 1]$, the noise vector is uniquely recoverable from the endpoints as $x_1 = \frac{1}{t}(x_t - M_t x_0)$. 
Substituting $x_1$ into $u_t$ yields an exact, simulation-free analytical target velocity
\begin{equation}
    u_t(x_t \mid x_0) = \Lambda_t x_0 + \frac{1}{t} x_t, \quad \text{where } \Lambda_t \triangleq \dot{M}_t - \frac{1}{t} M_t, 
    \label{eq:analytical_u}
\end{equation}
{and} $\Lambda_t \in \mathbb{R}^{D \times D}$ denotes the linear drift-velocity matrix governing the deterministic data evolution. 
Conditioned on the pair $(x_t, x_0)$, this target velocity is deterministic: $\operatorname{Var}(u_t \mid x_t, x_0) \equiv \mathbf{0}$, which nullifies the irreducible target noise floor ($\mathcal{E}_{\text{target}} = 0$ in Eq.~\eqref{eq:tripartite_loss}). 
Consequently, the flow matching loss directly penalizes the reducible representation error $\mathcal{E}_{\text{repr}}(\phi)$. 
Furthermore, because $B_t = t I_D$ is proportional to the identity matrix, it commutes with all subspace projectors ($[B_t, P] = \mathbf{0}$), ensuring that content and style trajectories evolve independently without noise cross-talk.

\subsection{Generative Curriculum and Style Truncation}
\label{sec:curriculum}

To ensure that the optimal representation isolates semantics, we parameterize the drift operator as 
\begin{equation}
    M_t = \gamma(t) P_c + \alpha(t) P_s.
\end{equation}
Macroscopic content follows the linear interpolation of optimal transport, while style undergoes a temporally truncated decay:
\begin{align} 
    \gamma(t) &=1 - t,\\
    \alpha(t) &= (1 - t) \, e^{-\kappa t} \cdot \mathbb{I}(t < t_1),
    \label{eq:style_schedule}
\end{align}
where $\kappa > 0$ is the decay rate and $t_1 \in (0, 1)$ sets the \textit{style truncation boundary}. 
Substituting these schedules into the drift operator $\Lambda_t$ yields decoupled scalar penalties across subspaces:
\begin{equation}
    \Lambda_t^T \Lambda_t = \lambda_c(t)^2 P_c + \lambda_s(t)^2 P_s.
    \label{eq:decoupled_penalty_operator}
\end{equation}
%
While the content penalty maintains steady $\mathcal{O}(1/t^2)$ precision across the full trajectory ($\lambda_c(t)^2 = 1/t^2$), the style penalty $\lambda_s(t)^2$, evaluated piecewise almost everywhere (a.e.) on $[0, 1] \setminus \{t_1\}$ with respect to the Lebesgue measure, is exponentially suppressed and vanishes identically for $t > t_1$ (detailed derivations in Appendix~\ref{app:penalty_derivation}).
As proven in Section~\ref{sec:theory}, setting $t_1 < 1$ suppresses the encoding of style features into the representation, bounding style variance reduction to a finite constant while allowing semantic variance reduction to diverge.

\subsection{Asymmetric Gradient Routing and Joint Objective}
\label{sec:routing}

As $t \to 0$, the state approaches clean data ($x_t \approx x_0$). 
Because $x_t$ contains nearly complete information about $x_0$, the encoder's  gradient vanishes, stalling feature learning. 
Conversely, the vector field regressor $v_\theta$ must train across the full interval $t \in [0, 1]$ to accurately calibrate sample generation at the clean data boundary \cite{zeng2025flowmatchinglownoiseregime, ho2020denoisingdiffusionprobabilisticmodels}.

To resolve this conflict, we implement \textit{asymmetric gradient routing} via a temporally gated stop-gradient operator $\operatorname{sg}[\cdot]$:
\begin{equation}
    z(t; \phi) = \mathbb{I}_{[t_0, 1]}(t) \, E_\phi(x_0) + \mathbb{I}_{[0, t_0)}(t) \, \operatorname{sg}\big[E_\phi(x_0)\big],
    \label{eq:gated_conditioning}
\end{equation}
where $t_0 \ll 1$ (e.g. $t_0 = 0.05$). 
Eq.~\eqref{eq:gated_conditioning} truncates backpropagation through $E_\phi$ near $t = 0$, shielding the encoder from vanishing updates while allowing $v_\theta$ to regress trajectories across the entire temporal domain.

To prevent dimensional collapse in the deterministic latent space, we regularize raw embeddings with sketched isotropic Gaussian regularization (SIGReg) \cite{balestriero2025lejepaprovablescalableselfsupervised}. 
The total optimization loss evaluates to:
\begin{equation}
    \mathcal{L}_{\text{CSFM}}(\theta, \phi) = \mathbb{E}_{t, x_0, x_1} \left[ \| u_t(x_t \mid x_0) - v_\theta(x_t, t, z(t; \phi)) \|_2^2 \right] + \lambda_{\text{reg}} \mathcal{L}_{\text{SIGReg}}(E_\phi(x_0)),
    \label{eq:total_csfm_objective}
\end{equation}
where $\lambda_{\text{reg}} \ge 0$ is a scalar hyperparameter balancing trajectory regression against latent geometric regularization.

\subsection{Simulation-Free Forward-Mode Vector Field Computation}
\label{sec:forward_computation}

In practice, the total degradation operator is an ordered cascade across $K$ physical domains:
\begin{equation}
    M_t = (1 - t) \prod_{k=1}^K C_t^{(k)}, \quad \text{where } C_t^{(k)} = P_c^{(k)} + e^{-\kappa_k t} \, \mathbb{I}(t < t_1) \, P_s^{(k)}.
    \label{eq:composite_M_cascade}
\end{equation}

{
Evaluating the time derivative $\dot{M}_t x_0$ via reverse-mode automatic differentiation requires caching computation graphs, incurring memory overhead. 
We instead evaluate $u_t(x_t \mid x_0)$ directly in forward mode.

For general non-commuting operators ($[C_t^{(i)}, C_t^{(j)}] \neq \mathbf{0}$), evaluation order is fixed, and $M_t x_0$ and $\dot{M}_t x_0$ are accumulated jointly via the sequential product rule in a single forward pass. 
However, for our proposed commuting operators, the cascade diagonalizes under the joint basis $Q = U_{\rm color} \otimes \mathcal{C}$ (Proposition~\ref{prop:joint_basis}):
\begin{equation}
    M_t = Q^T \boldsymbol{\Gamma}_t Q, \quad \text{with } \boldsymbol{\Gamma}_t \triangleq (1 - t) \prod_{k=1}^K \Lambda_t^{(k)},
    \label{eq:commuting_diagonal_M}
\end{equation}
where $\Lambda_t^{(k)} = Q C_t^{(k)} Q^T$ are diagonal matrices. 
In this regime, operator order is invariant, and the time derivative reduces to $\dot{M}_t = Q^T \dot{\boldsymbol{\Gamma}}_t Q$. 
Both $M_t x_0$ and $\dot{M}_t x_0$ are computed via coordinate-wise vector scaling in the transform domain.
Algorithms~\ref{alg:csfm_training} and~\ref{alg:csfm_target_flow} detail the complete implementation.}

\section{Theoretical Analysis: Optimal Representation Geometry}
\label{sec:theory}

In this section, we characterize the optimal encoder geometry induced by the CSFM objective. We show that under finite capacity, minimizing trajectory variance forces the encoder into a block-diagonal precision geometry that concentrates representational precision within the content subspace.

\subsection{Gaussian Surrogate Formulation and Exact Posterior}
\label{sec:theory_surrogate}

To characterize the representation geometry of the encoder, we formalize the subspace coordinates basis and the bottleneck capacity bounds.

\begin{definition}[subspace coordinate basis]
Recall the orthogonal decomposition $\mathbb{R}^D = \mathcal{X}_c \oplus \mathcal{X}_s$ from Section~\ref{sec:composite_subspaces} with projectors $P_c, P_s \in \mathbb{R}^{D \times D}$. 
We define an orthonormal basis matrix $V = [V_c, V_s] \in \mathbb{R}^{D \times D}$ such that $P_c = V_c V_c^T$ and $P_s = V_s V_s^T$, with dimensions $\dim(\mathcal{X}_c) = \operatorname{rank}(P_c) = c$ and $\dim(\mathcal{X}_s) = \operatorname{rank}(P_s) = s = D - c$.
\end{definition}

\begin{assumption}[maximum-entropy gaussian surrogate]
\label{ass:gaussian_surrogate}
For a deterministic embedding $z = E_\phi(x_0)$, the local conditional distribution $p(x_0 \mid z)$ is approximated by a regularized Gaussian surrogate $\mathcal{N}(\mu_{x \mid z}, \Sigma_\phi)$, where $\Sigma_\phi = (\Omega_\phi + \epsilon_0 I_D)^{-1} \in \mathbb{S}_{++}^D$, with $\Omega_\phi \in \mathbb{S}_+^D$ denoting the encoder precision matrix and $\epsilon_0 > 0$ an isotropic regularizer.
\end{assumption}

Because the multivariate Gaussian maximizes differential entropy {among all distributions with matched first and second moments, this surrogate models a worst-case bound on structural uncertainty; furthermore, the $L_2$ flow matching regression objective depends exclusively on conditional first and second moments \cite{cover2006elements, bishop2006pattern}.}

\begin{assumption}[bottleneck dimensionality and capacity bounds]
\label{ass:capacity_bounds}
The encoder architecture imposes two structural capacity constraints on $\Omega_\phi$ \cite{tishby1999information, bartlett2017spectrallynormalizedmarginboundsneural}:
\begin{enumerate}[label=(\roman*),leftmargin=*]
    \item \textbf{Dimensionality bottleneck:} the latent bottleneck dimension $c \ll D$ restricts the surrogate precision rank: $\operatorname{rank}(\Omega_\phi) \le c$ \cite{tishby1999information, alemi2016deep};
    \item \textbf{Spectral norm bound:} finite network weights enforce an upper bound on its maximum eigenvalue: $\|\Omega_\phi\|_2 = \lambda_{\max}(\Omega_\phi) \le \omega_{\max} < \infty$ \cite{bartlett2017spectrallynormalizedmarginboundsneural}.
\end{enumerate}
\end{assumption}

\paragraph{Bayes-Optimal Encoder Surrogate Objective.}
With the irreducible target noise nullified ($\mathcal{E}_{\text{target}} = 0$), substituting the analytical target velocity $u_t$ (Eq.~\eqref{eq:analytical_u}) into the tripartite decomposition (Eq.~\eqref{eq:tripartite_loss}) yields the asymptotic surrogate loss for the encoder as $v_\theta \to v^*$:
\begin{equation}
    \mathcal{E}_{\text{repr}}^*(\phi) = \int_0^1 w(t) \, \mathbb{E}_{x_t, z} \left[ \operatorname{Tr} \Big( \Lambda_t^T \Lambda_t \, \operatorname{Cov}(x_0 \mid x_t, z) \Big) \right] dt,
    \label{eq:enc_loss_formal}
\end{equation}
where $w(t)$ is the temporal sampling density and $\operatorname{Cov}(x_0 \mid x_t, z)$ denotes the posterior covariance under the surrogate model.

\begin{lemma}[closed-form posterior covariance]
\label{lem:posterior_covariance}
Under the linear observation model $x_t = M_t x_0 + t x_1$ with noise $x_1 \sim \mathcal{N}(0, I_D)$ and structural drift $M_t = \gamma(t) P_c + \alpha(t) P_s$, the posterior covariance $\Sigma_{\text{post}}(t) = \operatorname{Cov}(x_0 \mid x_t, z)$, under Assumption~\ref{ass:gaussian_surrogate}, is given for all $t \in (0, 1]$ by
\begin{equation}
    \Sigma_{\text{post}}(t) = t^2 \Big( t^2 \Omega_\phi + \big(\gamma(t)^2 + t^2 \epsilon_0\big) P_c + \big(\alpha(t)^2 + t^2 \epsilon_0\big) P_s \Big)^{-1}.
    \label{eq:sigma_post_closed}
\end{equation}
\end{lemma}
\noindent\textit{Proof Sketch.} The result follows from linear Gaussian conjugacy by summing the surrogate prior precision $\Omega_\phi$ and observation precision $\frac{1}{t^2} M_t^T M_t$. Expanding $M_t^T M_t$ via the projector orthogonality relations $P_c P_s = \mathbf{0}$ and $P_k^2 = P_k$ yields the bracketed form. Full details are provided in Appendix~\ref{app:proof_lemma_posterior}.

\subsection{Optimal Subspace Disentanglement}
\label{sec:theory_disentanglement}

Substituting the decoupled drift penalty $\Lambda_t^T \Lambda_t$ and the posterior covariance $\Sigma_{\text{post}}(t)$ (Lemma~\ref{lem:posterior_covariance}) into the encoder surrogate risk yields the instantaneous objective:
\begin{equation}
    \mathcal{E}_{t, \,\text{repr}}^*(\phi) = t^2 \operatorname{Tr} \left( \Big(\lambda_c(t)^2 P_c + \lambda_s(t)^2 P_s\Big) \Big( t^2 \Omega_\phi + \big(\gamma(t)^2 + t^2 \epsilon_0\big) P_c + \big(\alpha(t)^2 + t^2 \epsilon_0\big) P_s \Big)^{-1} \right).
    \label{eq:instantaneous_surrogate_loss}
\end{equation}

\begin{theorem}[exact subspace disentanglement]
\label{thm:disentanglement}
In the orthonormal basis $V = [V_c, V_s] \in \mathbb{R}^{D \times D}$ spanning $\mathcal{X}_c$ and $\mathcal{X}_s$, the encoder precision matrix partitions as
\begin{equation}
    \tilde{\Omega}_\phi = V^T \Omega_\phi V = \begin{pmatrix} \Omega_{cc} & \Omega_{cs} \\ \Omega_{sc} & \Omega_{ss} \end{pmatrix},
\end{equation}
where $\Omega_{cs} = \Omega_{sc}^T \in \mathbb{R}^{D_c \times D_s}$ denotes the cross-subspace coupling block. For any fixed diagonal blocks $\Omega_{cc}$ and $\Omega_{ss}$, the global infimum of the surrogate loss over $\Omega_\phi \in \mathbb{S}_+^D$ is achieved if and only if:
\begin{equation}
    \Omega_{cs} = \mathbf{0} \quad \iff \quad [\Omega_\phi^*, P_c] = \mathbf{0}.
\end{equation}
\end{theorem}
\noindent\textit{Proof Sketch.} Inverting the block precision matrix $\tilde{S}(t) = t^2 V^T \Sigma_{\text{post}}(t)^{-1} V$ via Schur complements shows that cross-coupling terms induce negative semi-definite shifts in the denominator: $[\tilde{S}(t)^{-1}]_{cc} = (S_{cc}(t) - \Delta_c(t))^{-1}$, where $\Delta_c(t) = t^4 \Omega_{cs} S_{ss}(t)^{-1} \Omega_{sc} \succeq \mathbf{0}$. By Loewner order monotonicity of matrix inversion, any non-zero coupling ($\Omega_{cs} \neq \mathbf{0}$) inflates the trace objective across all $t \in (0, 1)$. Hence, global optimality requires $\Omega_{cs} = \mathbf{0}$. See Appendix~\ref{app:proof_thm_disentanglement} for the complete proof.

\subsection{Subspace Capacity Allocation via Truncation}
\label{sec:theory_paradox}

By Theorem~\ref{thm:disentanglement}, the optimal precision matrix shares a common eigenbasis with the orthogonal projectors $P_c$ and $P_s$, diagonalizing as $\Omega_\phi^* = \sum_{j=1}^D \omega_j v_j v_j^T$, where each unit eigenvector satisfies either $v_j \in \mathcal{X}_c$ or $v_j \in \mathcal{X}_s$. Let the binary indicator $\delta_j = \|P_c v_j\|_2^2 \in \{0, 1\}$ denote content subspace membership ($\delta_j = 1$ if $v_j \in \mathcal{X}_c$, and $\delta_j = 0$ if $v_j \in \mathcal{X}_s$).

\begin{lemma}[capacity saturation]
\label{lem:saturation}
Under Assumption~\ref{ass:capacity_bounds}, the global surrogate objective is strictly monotonically decreasing with respect to each eigenvalue $\omega_j \in [0, \omega_{\max}]$. Therefore, an optimal encoder saturates its bottleneck capacity: exactly $c$ eigenvalues attain the maximum bound $\omega_j = \omega_{\max}$, while the remaining $D - c$ are identically zero.
\end{lemma}

Using Lemma~\ref{lem:saturation}, capacity allocation reduces to maximizing the integrated variance reduction $\Delta C = C(0) - C(\omega_{\max})$ over the active dimensions:
\begin{equation}
    \max_{\{\delta_j\}} \sum_{j=1}^c \Big[ \delta_j (\Delta C_c - \Delta C_s) + \Delta C_s \Big], \quad \delta_j \in \{0, 1\},
    \label{eq:capacity_linear_program}
\end{equation}
where the cumulative variance reductions for content ($\Delta C_c$) and style ($\Delta C_s$) evaluate to:
\begin{align}
    \Delta C_c = \int_0^1 w(t) \frac{\lambda_c(t)^2 t^4 \omega_{\max}}{[\gamma(t)^2 + t^2 \epsilon_0] \big[ t^2 (\omega_{\max} + \epsilon_0) + \gamma(t)^2 \big]} dt, \\
    \Delta C_s = \int_0^1 w(t) \frac{\lambda_s(t)^2 t^4 \omega_{\max}}{[\alpha(t)^2 + t^2 \epsilon_0] \big[ t^2 (\omega_{\max} + \epsilon_0) + \alpha(t)^2 \big]} dt.
    \label{eq:delta_C_definitions}
\end{align}

\begin{theorem}[subspace capacity allocation via horizon truncation]
\label{thm:truncation}
Let $\gamma(t) = 1 - t$, and let style decay exponentially with rate $\kappa > 0$ and truncation boundary $t_1 \in (0, 1]$: $\alpha(t) = (1 - t) e^{-\kappa t} \cdot \mathbb{I}(t < t_1)$:
\begin{enumerate}
    \item \textbf{Style domination under untruncated decay ($t_1 = 1$):} without truncation ($t_1 = 1$), style variance reduction exceeds content reduction ($\Delta C_s > \Delta C_c$) and the optimal encoder dedicates the bottleneck capacity to style ($\delta_j^* = 0$ for all $j \in \{1, \dots, c\}$).
    \item \textbf{Provable content alignment ($t_1 < 1$):} enforcing a hard truncation cutoff $t_1 < 1$ bounds $\Delta C_s$ to a finite constant, while $\Delta C_c$ diverges asymptotically as $t \to 1$. Consequently, $\Delta C_c \gg \Delta C_s$, compelling the optimal encoder to allocate all active dimensions to the content subspace ($\delta_j^* = 1$ for all $j \in \{1, \dots, c\}$).
\end{enumerate}
\end{theorem}
\noindent\textit{Proof Sketch.} In Case 1, as $t \to 1$, the style denominator vanishes exponentially faster than the content term, causing the style integrand to dominate pointwise for all $t > 0$. In Case 2, style degradation vanishes for $t \ge t_1$, truncating the domain of $\Delta C_s$ to $[0, t_1]$ where the integrand remains bounded by the extreme value theorem. In contrast, the content integrand integrates over the full unit interval to $t=1$, where $\lambda_c(t)^2 \sim \mathcal{O}(1/(1-t)^2)$ introduces a boundary singularity scaling as $\mathcal{O}(1/\sqrt{\epsilon_0}) \to \infty$ as $\epsilon_0 \to 0^+$ (or $\mathcal{O}(1/\delta_{\text{clamp}})$ under boundary clamping $t \le 1 - \delta_{\text{clamp}}$), guaranteeing $(\Delta C_c - \Delta C_s) > 0$. The complete calculus is provided in Appendix~\ref{app:proof_thm_truncation}.

\section{Experiments}
\label{sec:experiments}

\subsection{Experimental Setup}
\label{sec:exp_setup}

\paragraph{Datasets and Evaluation Protocols.}
We validate our framework across three visual benchmarks \cite{krizhevsky2009learning, coates2011analysis}:
{CIFAR-10} ($32 \times 32$, 10 classes) for primary benchmarking and ablations;
{CIFAR-100} ($32 \times 32$, 100 classes) for granular category scaling; and
{STL-10} ($96 \times 96$) for higher spatial resolution. For STL-10, unsupervised pre-training uses all 105k images (5k labeled + 100k unlabeled), while downstream probes evaluate on the standard 5k train and 8k test splits.
Frozen representations are evaluated via linear probing, non-parametric $k$-NN classification ($k=20$), and few-shot learning (1\% and 10\% labels per class) \cite{chen2020simpleframeworkcontrastivelearning, he2021maskedautoencodersscalablevision}. To measure latent cluster separation independently of classifiers, we report the silhouette score on test embeddings. Finally, secondary generative quality is evaluated in Appendix~\ref{app:extended_experiments} via $\text{FID}_{1\text{k}}$ \cite{heusel2017gans}, numerical ODE integration error, and latent space traversals.

\paragraph{Architectures and Capacity Scaling.}
To isolate decoupled transport dynamics from model size, backbones are benchmarked under multiple capacity controls:
\begin{itemize}[leftmargin=*]
    \item \textbf{Convolutional Models (NCSN++ U-Net):} The representation encoder is a WideResNet-28-4 ($E_\phi$, 5.98M parameters, latent dimension $c = 256$) \cite{zagoruyko2017wideresidualnetworks}. The vector field regressor ($v_\theta$) is evaluated under three regimes:
    \begin{enumerate}[label=(\roman*),leftmargin=*]
        \item \emph{Baseline (20.81M total):} Base filters $nf = 64$ with channel multipliers $(1, 2, 2, 2)$ and self-attention at $16 \times 16$ (14.83M U-Net parameters), serving as our primary ablation testbed.
        \item \emph{High-Capacity (65.08M total):} Expanded filters $nf = 128$ with FIR anti-aliasing (59.10M U-Net parameters). Encoder capacity remains fixed at 5.98M, ensuring generative gains reflect regressor capacity rather than expanded latent bandwidth.
        \item \emph{High-Resolution STL-10 (24.89M total):} A 18.91M parameter U-Net regressor (paired with the 5.98M encoder) fit for $96 \times 96$ feature maps under fixed memory limits.
    \end{enumerate}

    \item \textbf{Vision Transformers (ViT/DiT):} A custom ViT encoder ($E_\phi$, 5.62M parameters; patch size $4 \times 4$, embedding dimension 256, depth 6, 8 heads, MLP ratio 4.0, class-token pooling) paired with a continuous diffusion transformer regressor (\textsc{DiTContinuous}, 58.17M parameters, $\approx 63.8$M total) \cite{ma2024sitexploringflowdiffusionbased, peebles2023scalable}. The regressor comprises 12 DiT blocks (patch size $2 \times 2$, hidden dimension 384, 6 heads, MLP ratio 4.0). Continuous latents $z \in \mathbb{R}^{128}$ are projected via a 2-layer SiLU MLP and injected additively with timestep embeddings into adaptive layer norm (adaLN-Zero) blocks.
\end{itemize}

\paragraph{Implementation and Compute Constraints.}
All models are trained on a single NVIDIA L4 GPU (22.5GB VRAM) using mixed precision (\texttt{bfloat16}). We optimize CNNs with Adam \cite{kingma2014adam} and vision transformers with AdamW \cite{loshchilov2017decoupled}, both at a base learning rate of $\eta = 2 \times 10^{-4}$. Default hyperparameters are style truncation horizon $t_1 = 0.5$, asymmetric routing cutoff $t_0 = 0.05$, radial Gaussian cutoff $r_{\text{low}} = 6$ (Eq.~\eqref{eq:proj_freq}), and SIGReg weight $\lambda_{\text{reg}} = 10^{-3}$. Under single-GPU resource limits, all benchmarks are evaluated across standard fixed evaluation splits using a single consistent random seed. Detailed training schedules and hyperparameters are in Appendix~\ref{app:hyperparameters}.

\subsection{Representation Quality on CIFAR-10}
\label{sec:exp_representation_cifar10}

We evaluate semantic representations on frozen embeddings $z = E_\phi(x_0)$ via linear probing (100\%, 10\%, 1\% labeled splits), non-parametric $k$-NN ($k=20$), and the cluster silhouette score. 
CSFM is benchmarked against four representative paradigms trained under the composite loss $\mathcal{L}_{\text{total}} = \mathcal{L}_{\text{task}} + \lambda_{\text{reg}} \mathcal{L}_{\text{SIGReg}}(z)$: unconstrained autoencoding (AE) with disabled spatial bypasses, single-step conditional denoising (CDAE), linear optimal transport flow (OT-Flow), and joint-embedding prediction (LeJEPA \cite{balestriero2025lejepaprovablescalableselfsupervised}) with asymmetric dual views and a batch-normalized projector. 
We set $\lambda_{\text{reg}} = 10^{-2}$ for generative baselines and $\lambda_{\text{reg}} = 3 \times 10^{-2}$ for LeJEPA. 
Appendix~\ref{app:baseline_details} details the baseline objective functions, augmentation parameters, projector architectures, and regularization ablations.

\begin{table}[htpb]
\centering
\caption{Representation benchmark on CIFAR-10 across generative and discriminative paradigms (40K steps, baseline 20.81M tier: WRN-28-4 encoder paired with 14.83M U-Net). Best generative result in {bold}.}
\label{tab:cifar10_benchmark}
\renewcommand{\arraystretch}{1.12}
\begin{tabular}{@{}llccccc@{}}
\toprule
\multirow{2}{*}{\textbf{Paradigm}} & \multirow{2}{*}{\textbf{Model}} & \multicolumn{3}{c}{\textbf{Linear Probe Accuracy}} & \multirow{2}{*}{\textbf{$k$-NN}} & \multirow{2}{*}{\textbf{Silhouette}} \\
\cmidrule(lr){3-5}
& & \textbf{100\%} & \textbf{10\%} & \textbf{1\%} & & \\
\midrule
\multirow{2}{*}{\textit{Pixel Recon.}} 
& Autoencoder (AE) & 58.07\% & 51.08\% & 37.45\% & 48.37\% & -0.0780 \\
& CDAE (CSFM Path) & 57.27\% & 50.68\% & 36.06\% & 48.81\% & -0.0599 \\
\midrule
\multirow{3}{*}{\textit{Flow Matching}} 
& OT-Flow & 67.11\% & 62.25\% & 45.75\% & 52.97\% & -0.0454 \\
& CSFM ($t_1=1.0$, No Trunc.) & 69.59\% & 63.33\% & 46.31\% & 64.14\% & -0.0098 \\
& \textbf{CSFM ($t_1=0.5$, Trunc.)} & \textbf{75.29\%} & \textbf{71.94\%} & \textbf{58.93\%} & \textbf{68.34\%} & \textbf{-0.0154} \\
\midrule
\textit{Discriminative} 
& LeJEPA  & {72.49\%} & {71.02\%} & {67.76\%} & {{71.65\%}} & {{0.1738}} \\
\bottomrule
\end{tabular}
\end{table}

\paragraph{Flow Regression vs. Single-Step Denoising.}
Truncated CSFM ($t_1=0.5$) outperforms both pure autoencoding (58.07\%) and unguided linear flow (OT-Flow at 67.11\%). 
Crucially, CSFM outperforms CDAE (57.27\%) despite both operating over identical marginal degradation states $x_t$. 
Directly optimizing pixel reconstruction error ($\|x_0 - \hat{x}_\theta\|_2^2$) causes the encoder to expend capacity fitting high-frequency pixel variations \cite{balestriero2024learningreconstructionproducesuninformative, vanassel2025jointembeddingvsreconstruction}. 
Conversely, continuous velocity field regression ($u_t$) forces the bottleneck to encode persistent content features required to guide transport trajectories through time.

\paragraph{Bridging Generative Reversibility and Discriminative Invariance.}
Under a budget-constrained training regime of 40K iterations, CSFM achieves 75.29\% and 71.94\% linear accuracy with 100\% and 10\% labels, compared to 72.49\% and 71.02\% for LeJEPA. 
Conversely, within this computation budget, LeJEPA yields higher performance under 1\% labels (67.76\%) and neighborhood metrics ($k$-NN: 71.65\%, silhouette: 0.1738). 
This tradeoff reflects distinct geometric organizations: LeJEPA collapses intra-class variance to maximize margins early in training, contracting effective dimensionality ($d_{90\%} = 9$) and creating non-invertible low-density gaps between classes (Figure~\ref{fig:tsne_analysis}, right). 
Standard unguided OT-Flow maintains continuous density across classes but suffers from dispersed, overlapping cluster boundaries ($d_{90\%} = 31$). 
In contrast, truncated CSFM ($t_1=0.5$) limits style-subspace variance while retaining ambient coverage ($d_{90\%} = 37$), producing linearly separable semantic clusters alongside a globally connected latent manifold (Figure~\ref{fig:tsne_analysis}, middle). 
Appendix~\ref{app:latent_structure} provides the corresponding spectral and covariance analyses.

\begin{figure*}[t]
  \centering
  \includegraphics[width=\textwidth]{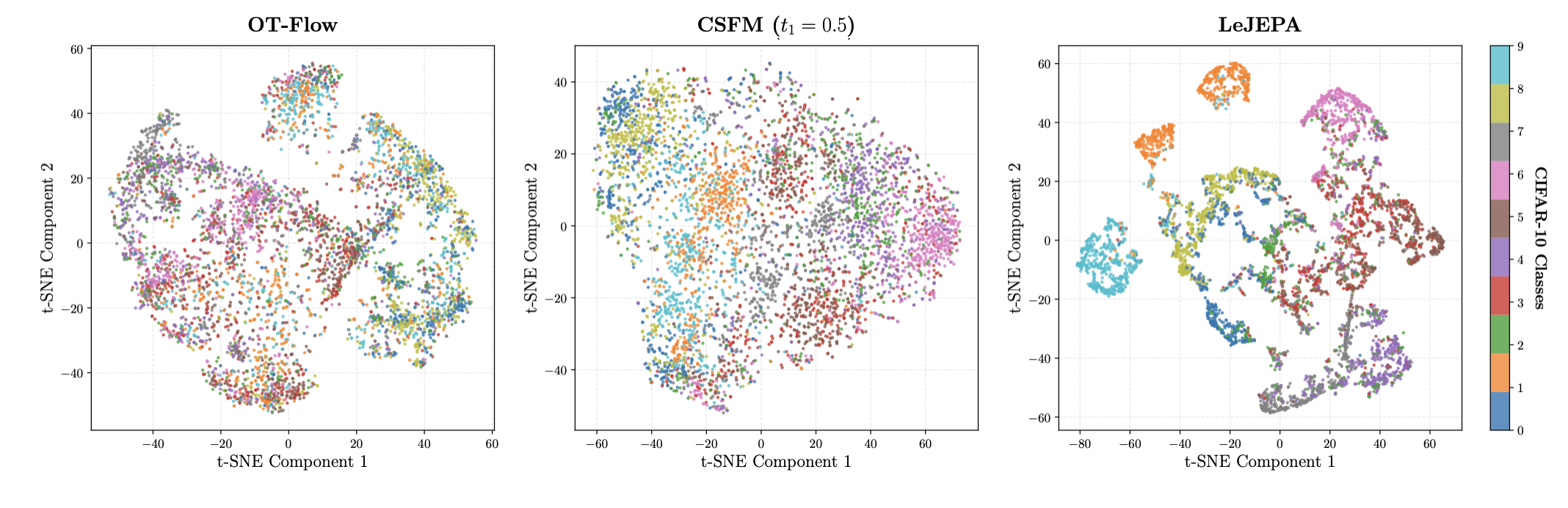}
  \caption{Topological manifold analysis via t-SNE embeddings on CIFAR-10 test set. (\textbf{Left}) OT-Flow forms dispersed, overlapping representations. (\textbf{Middle}) Truncated CSFM ($t_1=0.5$) isolates distinct semantic clusters while preserving global manifold connectivity. (\textbf{Right}) LeJEPA shatters representations into isolated topological islands.}
  \label{fig:tsne_analysis}
\end{figure*}

\paragraph{Decoder Capacity Scaling.}
In continuous flow representations, the vector field regressor $v_\theta$ acts as an implicit generative decoder for $z$. 
In Table~\ref{tab:scalability}, we scale the regressor to the high-capacity configuration (59.10M parameters) across an extended 120K step horizon while keeping the encoder capacity locked at 5.98M. 
Scaling $v_\theta$ boosts linear probing to 78.51\% and 1\% few-shot accuracy to 63.45\%. 
This confirms that a higher-capacity vector field exerts stronger variance-reduction pressure, driving the capacity-bounded encoder to extract more informative representations.

\begin{table}[htpb]
\centering
\caption{Decoder capacity scaling on CIFAR-10 with a fixed WideResNet-28-4 encoder (5.98M parameters) paired with the high-capacity U-Net regressor (59.10M parameters).}
\label{tab:scalability}
\renewcommand{\arraystretch}{1.1}
\begin{tabular}{@{}llcccc@{}}
\toprule
\textbf{Model} & \textbf{Steps} & \textbf{Linear (100\%)} & \textbf{1\% Few-Shot} & \textbf{$k$-NN} & \textbf{Silhouette} \\
\midrule
\multirow{3}{*}{\textbf{CSFM ($t_1=0.5$)}} 
& 40K & 77.31\% & \textbf{63.74\%} & \textbf{72.08\%} & \textbf{0.0068} \\
& 80K & \textbf{78.59\%} & 63.72\% & 71.77\% & -0.0038 \\
& 120K & 78.51\% & 63.45\% & 70.73\% & -0.0166 \\
\midrule
\multirow{3}{*}{CSFM ($t_1=1.0$)} 
& 40K & 69.86\% & 48.16\% & 64.54\% & -0.0099 \\
& 80K & 72.19\% & 49.01\% & 66.09\% & -0.0114 \\
& 120K & 73.59\% & 49.34\% & 66.74\% & -0.0154 \\
\midrule
\multirow{3}{*}{OT-Flow} 
& 40K & 68.19\% & 45.45\% & 53.81\% & -0.0433 \\
& 80K & 69.71\% & 45.57\% & 50.16\% & -0.0447 \\
& 120K & 70.48\% & 45.69\% & 48.17\% & -0.0534 \\
\bottomrule
\end{tabular}
\end{table}


\begin{table}[htpb]
\centering
\caption{Ablation of temporal boundaries on CIFAR-10 (20.81M baseline tier, 40K steps).}
\label{tab:ablation_horizons}
\renewcommand{\arraystretch}{1.1}
\begin{tabular}{@{}lcccc@{}}
\toprule
\textbf{Configuration} & \textbf{Linear Probe} & \textbf{1\% Few-Shot} & \textbf{$k$-NN} & \textbf{Silhouette} \\
\midrule
\multicolumn{5}{l}{\textit{Asymmetric Routing Cutoff ($t_0$) with $t_1=1.0$}} \\
$t_0 = 0.0$ (No Routing) & 69.09\% & 46.41\% & 63.13\% & -0.0185 \\
$t_0 = 0.001 $ & 69.12\% & 46.00\% & 63.04\% & -0.0202 \\
$t_0 = 0.025$ & 69.22\% & \textbf{47.84\%} & 64.23\% & -0.0126 \\
\textbf{$t_0 = 0.050$} & \textbf{69.52\%} & 47.41\% & \textbf{64.38\%} & \textbf{-0.0105} \\
\midrule
\multicolumn{5}{l}{\textit{Style Truncation Horizon ($t_1$) with $t_0=0.05$}} \\
$t_1 = 1.0$ (No Truncation) & 69.52\% & 47.41\% & 64.38\% & \textbf{-0.0105} \\
$t_1 = 0.8$ & 69.77\% & 47.49\% & 64.16\% & -0.0134 \\
$t_1 = 0.6$ & 72.22\% & 53.47\% & 66.12\% & -0.0257 \\
\textbf{$t_1 = 0.5$} & \textbf{75.12\%} & 57.40\% & \textbf{68.35\%} & -0.0215 \\
$t_1 = 0.4$ & 73.13\% & \textbf{57.70\%} & 64.99\% & -0.0185 \\
\bottomrule
\end{tabular}
\end{table}


\paragraph{Empirical Corroboration of Theoretical Horizons ($t_0, t_1$).}
Table~\ref{tab:ablation_horizons} evaluates the temporal integration boundaries on the 20.81M baseline tier, corroborating the theoretical analysis in Section~\ref{sec:theory}. 
Severing the computational graph near $t \to 0$ with $t_0 = 0.05$ shields the encoder from asymptotic gradient vanishing, improving linear probe accuracy from 69.09\% to 69.52\%. 
More critically, introducing a style cutoff at $t_1 = 0.5$ corroborates Theorem~\ref{thm:truncation} by mitigating style-subspace domination: it flips the variance reduction inequality ($\Delta C_s > \Delta C_s$), increasing linear probing accuracy to 75.12\% and 1\% few-shot accuracy to 57.40\% as capacity shifts from high-frequency style to content.

\paragraph{Physical Decay Dynamics and Latent Regularization.}
Extended ablations in Appendix~\ref{app:extended_experiments} corroborate the remaining theoretical remarks. 
Under optimal truncation ($t_1 = 0.5$), representation quality remains stable across three orders of magnitude of the base decay rate ($\kappa \in [10^{-4}, 6.0]$), consistently maintaining $>75\%$ linear accuracy (Table~\ref{tab:app_ablation_physics}). 
This corroborates Theorem~\ref{thm:truncation} by confirming that the discrete truncation boundary $t_1$, rather than continuous path curvature $\kappa$, governs content isolation. 
Conversely, without truncation ($t_1 = 1.0$), tuning $\kappa$ and $\beta$ yields negligible gains ($\le 69.60\%$), confirming that continuous parameter adjustments cannot substitute for boundary truncation. 
Finally, ablating the regularization weight in Table~\ref{tab:app_ablation_reg} corroborates Eq.~\eqref{eq:variance_reduction_bound}: while omitting regularization ($\lambda_{\text{reg}} = 0$) does not induce representational collapse, incorporating the constraint at $\lambda_{\text{reg}} = 10^{-3}$ yields measurable performance gains (69.09\% linear accuracy), whereas over-penalizing ($\lambda_{\text{reg}} = 0.1$) degrades downstream separability.

\subsection{Scalability Analysis: Cardinality, Architecture, and Resolution}
\label{sec:exp_scalability}

To establish generality beyond CIFAR-10, we evaluate CSFM across three scaling dimensions under controlled parameter budgets within each tier (Table~\ref{tab:scalability_benchmark}):
\begin{enumerate}[label=(\roman*),leftmargin=*]
    \item {Semantic cardinality (CIFAR-100):} scaling category complexity tenfold ($10 \to 100$) while reducing training data to 500 images per class.
    \item {Architectural inductive bias (ViT / DiT):} replacing convolutional inductive biases with patch tokenization and global multi-head self-attention.
    \item {Spatial resolution scaling (STL-10):} scaling from $32 \times 32$ to  $96 \times 96$ pixels on an uncurated 105k image pool.
\end{enumerate}

\begin{table*}[t]
\centering
\caption{Scalability benchmarks across semantic cardinality (CIFAR-100), backbones (CNN vs.\ ViT/DiT), and spatial resolution (STL-10). Evaluated under strict parameter parity within each tier. Best results per tier in {bold}.}
\label{tab:scalability_benchmark}
\renewcommand{\arraystretch}{1.15}
\begin{tabular}{@{}lllcccc@{}}
\toprule
\textbf{Benchmark Axis} & \textbf{Dataset} & \textbf{Model Backbone} & \textbf{Steps} & \textbf{Linear Probe} & \textbf{1\% Few-Shot} & \textbf{$k$-NN ($k=20$)} \\
\midrule
\multirow{2}{*}{\textit{Semantic Cardinality}} & CIFAR-100 & OT-Flow (CNN) & 40K & 42.27\% & 10.66\% & 33.63\% \\
& CIFAR-100 & \textbf{CSFM (CNN)} & 40K & \textbf{51.58\%} & \textbf{17.65\%} & \textbf{38.97\%} \\
\midrule
\multirow{4}{*}{\textit{Architectural Bias}} 
& CIFAR-10 & OT-Flow (ViT/DiT) & 40K & 51.84\% & 39.15\% & 56.81\% \\
& CIFAR-10 & OT-Flow (ViT/DiT) & 80K & 53.27\% & 41.06\% & 58.36\% \\
\cmidrule(lr){3-7}
& CIFAR-10 & CSFM (ViT/DiT) & 40K & 53.30\% & 42.27\% & 55.14\% \\
& CIFAR-10 & \textbf{CSFM (ViT/DiT)} & 80K & \textbf{61.49\%} & \textbf{49.69\%} & \textbf{62.77\%} \\
\midrule
\multirow{2}{*}{\textit{Spatial Resolution}} 
& STL-10 & OT-Flow (CNN) & 40K & 63.90\% & 48.34\% & 55.62\% \\
& STL-10 & \textbf{CSFM (CNN)} & 40K & \textbf{65.74\%} & \textbf{50.54\%} & \textbf{57.91\%} \\
\bottomrule
\end{tabular}
\end{table*}

\paragraph{Semantic Cardinality: CIFAR-100.}
Evaluating on CIFAR-100 tests representation robustness under higher label granularity and limited per-class supervision. 
Standard OT-Flow achieves 42.27\% linear accuracy and drops to 10.66\% with 1\% labeled data (5 samples per class). 
In contrast, CSFM attains {51.58\%} linear accuracy ($+9.31\%$) and {17.65\%} on 1\% labels ($+6.99\%$). 
While uniform OT perturbation forces the regressor to resolve high-frequency noise across all scales, CSFM suppresses style variance along the trajectory, preserving linear separability in fine-grained regimes.

\paragraph{Architectural Scalability: Vision Transformers (ViT/DiT).}
Evaluating transformer backbones demonstrates that CSFM is architecture-agnostic. 
CSFM scales from $53.30\%$ at 40K steps to {61.49\%} at 80K steps ($+8.19\%$), while standard OT-Flow yields diminishing returns ($51.84\% \to 53.27\%$). 
On 1\% few-shot labels, CSFM achieves {49.69\%}, outperforming OT-Flow by $+8.63\%$. 
Transformers lack local translational inductive biases and typically require aggressive data augmentation; CSFM’s structured trajectories provide a stabilizing generative prior that guides token representations.

\paragraph{Spatial Resolution Scaling: STL-10 ($96 \times 96$).}
Benchmarking native $96 \times 96$ images on STL-10 tests scalability to higher pixel dimensions. 
Due to GPU memory constraints at this resolution, we employ a lighter, lower-capacity U-Net decoder ($\approx$ 19M params) and reduce the batch size from 64 to 16; as established in our CIFAR-10 capacity ablations, reduced decoder capacity inherently lowers representation quality.
Nevertheless, pre-training on the 105k uncurated images and evaluating frozen features on the 5k labeled split yields consistent gains over OT-flow: CSFM achieves 65.74\% linear accuracy ($+1.84\%$), {50.54\%} on 1\% labels ($+2.20\%$), and {57.91\%} on $k$-NN ($+2.29\%$). 
This confirms that physical subspace decomposition scales to higher image resolutions without sacrificing content selectivity.

\section{Conclusion}
\label{sec:conclusion}

In this paper, we introduced composite subspace flow matching (CSFM), a generative framework that repurposes continuous normalizing flows for representation learning. 
By pairing deterministic drift across canonical visual subspaces with a decoupled noise path, CSFM eliminates the irreducible target variance floor of multi-source stochastic diffusion while evaluating in $\mathcal{O}(D \log D)$ time. 
We showed that untruncated exponential degradation causes encoders to allocate capacity to high-frequency style details, and proved that a finite style cutoff ($t_1 < 1$) enforces strict feature allocation toward the content subspace $\mathcal{X}_c$. 
Empirically, CSFM  outperforms unguided flows and pixel-level denoising autoencoders across CIFAR-10, CIFAR-100, and STL-10. Under resource-limited training schedules, it proves competitive with discriminative joint-embedding baselines on linear probes while scaling across convolutional and transformer architectures and retaining generative invertibility.

\section*{Code Availability}

Complete code and experimental configurations are publicly available at \url{https://github.com/joseAf28/csfm-paper}

\bibliographystyle{unsrtnat}
\bibliography{references}

\appendix

\section{Omitted Derivations from Section 3}

\subsection{Simultaneous Diagonalization of Physical Subspace Operators}
\label{app:simultaneous_diagonalization}

We show that the color, spectral frequency, and illumination operators defined in Section~\ref{sec:composite_subspaces} mutually commute and admit a common orthonormal eigenbasis. 
Recall that the observation space factors as $\mathbb{R}^D \cong \mathbb{R}^3 \otimes \mathbb{R}^N$ ($D = 3N$), where $\mathbb{R}^3$ denotes the color space and $\mathbb{R}^N$ ($N = HW$) corresponds to the spatial component. 
Because the style operators are orthogonal complements ($P_s = I_D - P_c$), an operator commutes with $P_c$ if and only if it commutes with $P_s$, since $[A, I_D - P_c] = -[A, P_c]$. 
Thus, establishing commutativity for the content projectors $P_c$ suffices for all associated style projectors.

\begin{proposition}
\label{prop:joint_basis}
The operators $P_c^{\rm color}$, $P_c^{\rm freq}$, and $P_c^{\rm illum}$ are mutually commuting symmetric operators:
\begin{equation}
    [P_c^{\rm color}, P_c^{\rm freq}] = [P_c^{\rm color}, P_c^{\rm illum}] = [P_c^{\rm freq}, P_c^{\rm illum}] = \mathbf{0}.
    \label{eq:pairwise_commutation}
\end{equation}
Consequently, they are simultaneously diagonalized by the separable orthonormal basis:
\begin{equation}
    Q = U_{\rm color} \otimes \mathcal{C} \in \mathbb{R}^{D \times D},
    \label{eq:joint_basis_matrix}
\end{equation}
where $U_{\rm color} \in \mathbb{R}^{3 \times 3}$ is the orthogonal Ohta matrix and $\mathcal{C} = C_W \otimes C_H \in \mathbb{R}^{N \times N}$ is the 2D orthonormal DCT-II basis. 
Any linear composition of these operators evaluates in $\mathcal{O}(D \log D)$ time.
\end{proposition}

\begin{proof}
We express each operator over the tensor product space $\mathbb{R}^3 \otimes \mathbb{R}^N$:
\begin{align}
    P_c^{\rm color} &= \Pi_{\rm color} \otimes I_N, 
    &&\text{where } \Pi_{\rm color} = U_{\rm color}^T \operatorname{diag}(1, 0, 0) U_{\rm color} = u_{\rm DC} u_{\rm DC}^T \in \mathbb{R}^{3 \times 3}, \\
    P_c^{\rm freq} &= I_3 \otimes \Pi_{\rm freq}, 
    &&\text{where } \Pi_{\rm freq} = \mathcal{C}^T \Lambda_{\rm freq} \mathcal{C} \in \mathbb{R}^{N \times N}, \\
    P_c^{\rm illum} &= I_D - \Pi_{\rm color} \otimes \Pi_w, 
    &&\text{where } \Pi_w = w_{\rm DC} w_{\rm DC}^T \in \mathbb{R}^{N \times N},
\end{align}
with $u_{\rm DC} = \frac{1}{\sqrt{3}}\mathbf{1}_3$, $w_{\rm DC} = \frac{1}{\sqrt{N}}\mathbf{1}_N$, and $\Lambda_{\rm freq} = \operatorname{diag}(\{P_c^{\rm freq}(u, v)\}_{u, v})$. 
The term $v_{\rm DC} = u_{\rm DC} \otimes w_{\rm DC}$ reflects the normalized spatial-color DC vector, so that $P_s^{\rm illum} = v_{\rm DC} v_{\rm DC}^T = \Pi_{\rm color} \otimes \Pi_w$.

Pairwise commutativity follows from the Kronecker mixed-product identity $(A_1 \otimes B_1)(A_2 \otimes B_2) = (A_1 A_2) \otimes (B_1 B_2)$. 
For color and frequency, the operators act on disjoint tensor factors:
\begin{equation}
    P_c^{\rm color} P_c^{\rm freq} = (\Pi_{\rm color} I_3) \otimes (I_N \Pi_{\rm freq}) = \Pi_{\rm color} \otimes \Pi_{\rm freq} = P_c^{\rm freq} P_c^{\rm color}.
\end{equation}
For color and illumination, commutativity reduces to the color factor, where $\Pi_{\rm color}$ is idempotent ($\Pi_{\rm color}^2 = \Pi_{\rm color}$):
\begin{equation}
    P_c^{\rm color} (I_D - P_c^{\rm illum}) = (\Pi_{\rm color}^2) \otimes (I_N \Pi_w) = \Pi_{\rm color} \otimes \Pi_w = (I_D - P_c^{\rm illum}) P_c^{\rm color},
\end{equation}
implying $[P_c^{\rm color}, P_c^{\rm illum}] = \mathbf{0}$. 

Finally, for frequency and illumination, the zero-frequency 2D DCT-II basis function is constant, so $\mathcal{C} w_{\rm DC} = e_{(0,0)}$, where $e_{(0,0)} \in \mathbb{R}^N$ denotes the canonical basis unit vector corresponding to the zero-frequency $(u, v) = (0, 0)$ coordinate. 
The projected spatial DC operator is diagonal in the transform domain: $\mathcal{C} \Pi_w \mathcal{C}^T = e_{(0,0)} e_{(0,0)}^T = \operatorname{diag}(1, 0, \dots, 0)$. 
Because both $\Lambda_{\rm freq}$ and $\mathcal{C} \Pi_w \mathcal{C}^T$ are diagonal matrices in the DCT basis, they commute:
\begin{equation}
    [\Pi_{\rm freq}, \Pi_w] = \mathcal{C}^T [\Lambda_{\rm freq}, \mathcal{C} \Pi_w \mathcal{C}^T] \mathcal{C} = \mathbf{0},
\end{equation}
which establishes $[P_c^{\rm freq}, P_c^{\rm illum}] = \mathbf{0}$.

Because all three symmetric operators commute pairwise, the spectral theorem guarantees simultaneous diagonalization by $Q = U_{\rm color} \otimes \mathcal{C}$. 
In this joint eigenbasis, each operator simplifies to a diagonal matrix:
\begin{align}
    Q P_c^{\rm color} Q^T &= \operatorname{diag}(1, 0, 0) \otimes I_N, \\
    Q P_c^{\rm freq} Q^T &= I_3 \otimes \Lambda_{\rm freq}, \\
    Q P_c^{\rm illum} Q^T &= I_D - \operatorname{diag}(1, 0, 0) \otimes \operatorname{diag}(1, 0, \dots, 0).
\end{align}
Evaluating transformations under $Q$ requires only channel-wise 2D fast cosine transforms, ensuring that any composition, power, or inverse covariance filtering executes in $\mathcal{O}(D \log D)$ time.
\end{proof}

\subsection{Derivation of the Tripartite Variance Decomposition}
\label{app:tripartite_derivation}

We derive the tripartite loss decomposition in Eq.~\eqref{eq:tripartite_loss} for the conditional objective:
\begin{equation}
    \mathcal{L}_{\text{CFM}}(\theta, \phi) = \mathbb{E}_{t, x_0, x_1} \left\| v_\theta(x_t, t, z) - u_t \right\|_2^2,
\end{equation}
where $t \sim \mathcal{U}[0, 1]$, $x_0 \sim p_{\text{data}}$, $x_1 \sim \mathcal{N}(0, I_D)$, and $z = E_\phi(x_0)$.

Decomposing around the optimal predictor $v^*(x_t, t, z) \triangleq \mathbb{E}[u_t \mid x_t, t, z]$:
\begin{equation}
    \| v_\theta(x_t, t, z) - u_t \|_2^2 = \| v_\theta - v^* \|_2^2 + \| v^* - u_t \|_2^2 + 2 \langle v_\theta - v^*, v^* - u_t \rangle.
\end{equation}
Conditioning on $(x_t, t, z)$, the cross-term vanishes:
\begin{equation}
    \mathbb{E}\big[ \langle v_\theta - v^*, v^* - u_t \rangle \;\big|\; x_t, t, z \big] = \langle v_\theta - v^*, v^* - \mathbb{E}[u_t \mid x_t, t, z] \rangle = 0.
\end{equation}
Taking total expectations yields:
\begin{equation}
    \mathcal{L}_{\text{CFM}}(\theta, \phi) = \underbrace{\mathbb{E} \| v_\theta(x_t, t, z) - v^*(x_t, t, z) \|_2^2}_{\mathcal{E}_{\text{model}}(\theta)} + \mathbb{E} \| u_t - \mathbb{E}[u_t \mid x_t, t, z] \|_2^2.
    \label{eq:app_step1}
\end{equation}

The second term in Eq.~\eqref{eq:app_step1} is the trace of the conditional covariance:
\begin{equation}
    \mathbb{E} \| u_t - \mathbb{E}[u_t \mid x_t, t, z] \|_2^2 = \mathbb{E}_{t, x_t, z} \left[ \operatorname{Tr}\Big(\operatorname{Var}(u_t \mid x_t, t, z)\Big) \right].
\end{equation}
Since $z = E_\phi(x_0)$, the sub-$\sigma$-algebras satisfy $\sigma(x_t, t, z) \subseteq \sigma(x_t, t, x_0)$. Applying the law of total variance conditioned on $(x_t, t, z)$ with respect to $x_0$:
\begin{equation}
    \operatorname{Var}(u_t \mid x_t, t, z) = \operatorname{Var}_{x_0 \mid x_t, t, z}\Big( \mathbb{E}[u_t \mid x_t, t, x_0] \Big) + \mathbb{E}_{x_0 \mid x_t, t, z}\Big[ \operatorname{Var}(u_t \mid x_t, t, x_0) \Big].
\end{equation}
Taking the trace and outer expectations over $(t, x_t, z)$ yields:
\begin{equation}
    \mathbb{E} \| u_t - \mathbb{E}[u_t \mid x_t, t, z] \|_2^2 = \underbrace{\mathbb{E} \left[ \operatorname{Tr}\Big(\operatorname{Var}_{x_0 \mid x_t, z}\big(\mathbb{E}[u_t \mid x_t, x_0]\big)\Big) \right]}_{\mathcal{E}_{\text{repr}}(\phi)} + \underbrace{\mathbb{E} \left[ \operatorname{Tr}\Big(\operatorname{Var}(u_t \mid x_t, x_0)\Big) \right]}_{\mathcal{E}_{\text{target}}}.
\end{equation}
Substituting back into Eq.~\eqref{eq:app_step1} completes the derivation of Eq.~\eqref{eq:tripartite_loss}. \qed

\subsection{Derivation of the Cumulative Variance Floor}
\label{app:cumulative_limitations}
Consider $x_t = M_t x_0 + \mathbf{B}_t \boldsymbol{\epsilon}$, where $\boldsymbol{\epsilon} \sim \mathcal{N}(0, I_{KD})$ and $\mathbf{B}_t \in \mathbb{R}^{D \times KD}$ with $KD > D$. 
Given $(x_t, x_0)$, the posterior distribution of $\boldsymbol{\epsilon}$ is Gaussian with conditional covariance:
\begin{equation}
    \operatorname{Cov}(\boldsymbol{\epsilon} \mid x_t, x_0) = I_{KD} - \mathbf{B}_t^T (\mathbf{B}_t \mathbf{B}_t^T)^{-1} \mathbf{B}_t.
\end{equation}
Because $u_t = \dot{M}_t x_0 + \dot{\mathbf{B}}_t \boldsymbol{\epsilon}$, the residual target velocity variance evaluates to:
\begin{equation}
    \operatorname{Var}(u_t \mid x_t, x_0) = \dot{\mathbf{B}}_t \Big( I_{KD} - \mathbf{B}_t^T (\mathbf{B}_t \mathbf{B}_t^T)^{-1} \mathbf{B}_t \Big) \dot{\mathbf{B}}_t^T \succ \mathbf{0},
\end{equation}
which is non-zero whenever $\dot{\mathbf{B}}_t$ has components in the null space of $\mathbf{B}_t$.

\subsection{Derivation of Decoupled Penalty Weights}
\label{app:penalty_derivation}

Substituting the orthogonal subspace decomposition $M_t = \gamma(t) P_c + \alpha(t) P_s$ into the velocity drift matrix $\Lambda_t = \dot{M}_t - \frac{1}{t} M_t$ gives:
\begin{equation}
    \Lambda_t = \underbrace{\left(\dot{\gamma}(t) - \frac{\gamma(t)}{t}\right)}_{\lambda_c(t)} P_c + \underbrace{\left(\dot{\alpha}(t) - \frac{\alpha(t)}{t}\right)}_{\lambda_s(t)} P_s \equiv \lambda_c(t) P_c + \lambda_s(t) P_s.
\end{equation}
Because $P_c$ and $P_s$ are orthogonal projectors ($P_c P_s = \mathbf{0}$ and $P_k^2 = P_k$), cross-terms vanish under transposition, yielding a block-diagonal squared drift operator:
\begin{equation}
    \Lambda_t^T \Lambda_t = \lambda_c(t)^2 P_c + \lambda_s(t)^2 P_s.
\end{equation}

For the linear decay schedule $\gamma(t) = 1 - t$, differentiation yields $\dot{\gamma}(t) = -1$. Evaluating $\lambda_c(t)$ gives:
\begin{equation}
    \lambda_c(t) = -1 - \frac{1 - t}{t} = -\frac{1}{t} \implies \lambda_c(t)^2 = \frac{1}{t^2}.
\end{equation}

For the truncated schedule $\alpha(t) = (1 - t) e^{-\kappa t} \, \mathbb{I}(t < t_1)$, the derivative exists on $(0, 1) \setminus \{t_1\}$. 
On $t \in (0, t_1)$, applying the product rule yields:
\begin{equation}
    \dot{\alpha}(t) = -e^{-\kappa t} - \kappa (1 - t) e^{-\kappa t} = -e^{-\kappa t} \big( 1 + \kappa(1 - t) \big).
\end{equation}
The instantaneous style coefficient on $t \in (0, t_1)$ evaluates to:
\begin{align}
    \lambda_s(t) &= -e^{-\kappa t} \big( 1 + \kappa(1 - t) \big) - \frac{(1 - t)e^{-\kappa t}}{t} = -\frac{e^{-\kappa t}}{t} \Big[ 1 + \kappa t(1 - t) \Big].
\end{align}
On the upper interval $t \in (t_1, 1)$, $\alpha(t) \equiv 0$ and $\dot{\alpha}(t) \equiv 0$, identically yielding $\lambda_s(t) \equiv 0$. 

At the truncation boundary $t = t_1$, $\alpha(t)$ exhibits a jump discontinuity. Because training time $t$ is drawn from an absolutely continuous density $w(t)$ with respect to the Lebesgue measure, the isolated point $\{t_1\}$ has measure zero ($\mathbb{P}(t = t_1) = 0$). Evaluating $\lambda_s(t)^2$ piecewise almost everywhere (a.e.) on $(0, 1) \setminus \{t_1\}$ yields:
\begin{equation}
    \lambda_s(t)^2 = \frac{e^{-2\kappa t}}{t^2} \Big[ 1 + \kappa t(1 - t) \Big]^2 \cdot \mathbb{I}(t < t_1), \quad \text{a.e. } t \in (0, 1).
    \label{eq:penalty_weights}
\end{equation}
Since $\{t_1\}$ carries zero probability mass, this representation integrates to the exact expected surrogate loss $\mathcal{E}_{\text{repr}}^*(\phi)$ without delta contributions.

\section{Algorithmic Workflows}
\label{app:algorithms}

We summarize the core algorithmic routines below. 
Algorithm~\ref{alg:csfm_training} details the joint training pipeline with asymmetric gradient routing and SIGReg regularization. 
Algorithm~\ref{alg:csfm_target_flow} provides the forward-mode computation of $M_t x_0$ and $\dot{M}_t x_0$ across the cascade. 
Algorithm~\ref{alg:csfm_inference} describes reverse-time generative sampling across conditional and unconditional settings.

\begin{figure*}[t]
\centering
\begin{minipage}[t]{0.485\textwidth}
\begin{algorithm}[H]
\caption{CSFM Training Objective}
\label{alg:csfm_training}
\begin{algorithmic}[1]
\Require Dataset $\mathcal{D}$, Encoder $E_\phi$, Vector field $v_\theta$, Flow engine $\mathcal{F}$, Horizons $(t_0, t_1)$, Weight $\lambda_{\text{reg}}$
\While{not converged}
    \State Sample batch $x_0 \sim \mathcal{D}$ and continuous time $t \sim \mathcal{U}(0, 1)$
    
    \vspace{1.5mm}
    \Statex \textcolor{gray}{\# 1. Target Generation \& Representation}
    \State $x_t, u_t \gets \mathcal{F}(x_0, t, t_1)$ \Comment{See Alg.~\ref{alg:csfm_target_flow}}
    \State $z_{\text{raw}} \gets E_\phi(x_0)$
    
    \vspace{1.5mm}
    \Statex \textcolor{gray}{\# 2. Asymmetric Gradient Routing}
    \State $m \gets \mathbb{I}(t \ge t_0)$ \Comment{Routing mask}
    \State $z_{\text{routed}} \gets m \cdot z_{\text{raw}} + (1 - m) \cdot \operatorname{sg}(z_{\text{raw}})$
    
    \vspace{1.5mm}
    \Statex \textcolor{gray}{\# 3. Vector Field Regression}
    \State $v_{\text{pred}} \gets v_\theta(x_t, t, z_{\text{routed}})$
    \State $\mathcal{L}_{\text{flow}} \gets \| v_{\text{pred}} - u_t \|_2^2$
    
    \vspace{1.5mm}
    \Statex \textcolor{gray}{\# 4. Latent Space Regularization}
    \State $\mathcal{L}_{\text{reg}} \gets \operatorname{SIGReg}(z_{\text{raw}})$
    
    \vspace{1.5mm}
    \Statex \textcolor{gray}{\# 5. Joint Parameter Update}
    \State $\mathcal{L}_{\text{total}} \gets \mathcal{L}_{\text{flow}} + \lambda_{\text{reg}} \mathcal{L}_{\text{reg}}$
    \State Update $\phi, \theta$ via gradient descent on $\nabla \mathcal{L}_{\text{total}}$
\EndWhile
\end{algorithmic}
\end{algorithm}
\end{minipage}\hfill
\begin{minipage}[t]{0.485\textwidth}
\begin{algorithm}[H]
\caption{CSFM Analytical Target Vector Flow}
\label{alg:csfm_target_flow}
\begin{algorithmic}[1]
\Require Sample $x_0$, Time $t$, Horizon $t_1$, Operators $\mathcal{O}$, Decay rates $\boldsymbol{\kappa}$, Commuting flag $\mathtt{commute}$
\Ensure Target state $x_t$, Target velocity $u_t$

\vspace{1.5mm}
\Statex \textcolor{gray}{\# 1. Forward-Mode Evaluation}
\If{$\mathtt{commute}$}
    \State $\hat{x}_0 \gets Q x_0$
    \State $\boldsymbol{\lambda}_t, \dot{\boldsymbol{\lambda}}_t \gets \operatorname{GetSpectralRates}(t, \boldsymbol{\kappa}, t_1)$
    \State $C_t(x_0) \gets Q^T (\boldsymbol{\lambda}_t \odot \hat{x}_0), \quad \dot{C}_t(x_0) \gets Q^T (\dot{\boldsymbol{\lambda}}_t \odot \hat{x}_0)$
\Else
    \State $s \gets x_0, \quad ds \gets \mathbf{0}$
    \For{$k = 1$ \textbf{to} $K$}
        \State $\gamma, \alpha, \dot{\gamma}, \dot{\alpha} \gets \operatorname{GetRates}(t, \boldsymbol{\kappa}[k], t_1)$
        \State $s, ds \gets \mathcal{O}[k].\text{apply}(s, ds, \gamma, \alpha, \dot{\gamma}, \dot{\alpha})$
    \EndFor
    \State $C_t(x_0) \gets s, \quad \dot{C}_t(x_0) \gets ds$
\EndIf

\vspace{1.5mm}
\Statex \textcolor{gray}{\# 2. $M_t$ and $\dot{M}_t$ Calculations}
\State $M_t(x_0) \gets (1 - t) \cdot C_t(x_0)$
\State $\dot{M}_t(x_0) \gets -C_t(x_0) + (1 - t) \cdot \dot{C}_t(x_0)$

\vspace{1.5mm}
\Statex \textcolor{gray}{\# 3. Flow Matching Target Formulation}
\State Sample spatial noise $x_1 \sim \mathcal{N}(0, I_D)$
\State $x_t \gets M_t(x_0) + t \cdot x_1$
\State $u_t \gets \dot{M}_t(x_0) + x_1$

\vspace{1.5mm}
\State \Return $x_t, u_t$
\end{algorithmic}
\end{algorithm}
\end{minipage}
\end{figure*}

\begin{algorithm}[t]
\caption{CSFM Dual-Mode Generative Inference}
\label{alg:csfm_inference}
\begin{algorithmic}[1]
\Require Encoder $E_\phi$, Vector field $v_\theta$, Step size $h$, Steps $N = 1/h$, Optional reference $x_0^*$
\Ensure Synthesized image $\hat{x}_0$

\vspace{1.5mm}
\Statex \textcolor{gray}{\# 1. Semantic Initialization}
\If{$x_0^*$ is provided}
    \State $z^* \gets E_\phi(x_0^*)$ \Comment{Conditional semantics}
\Else
    \State $z^* \sim \mathcal{N}(0, I_c)$ \Comment{Unconditional prior sampling}
\EndIf

\vspace{1.5mm}
\Statex \textcolor{gray}{\# 2. Spatial Initialization}
\State Sample $x_1 \sim \mathcal{N}(0, I_D)$ \Comment{Prior boundary at $t = 1.0$}
\State $\hat{x}_N \gets x_1$
\State $t \gets 1.0$

\vspace{1.5mm}
\Statex \textcolor{gray}{\# 3. Reverse ODE Integration}
\For{$i = N$ \textbf{downto} $1$}
    \State $d_i \gets v_\theta(\hat{x}_i, t, z^*)$ \Comment{Evaluate vector field}
    \State $\hat{x}_{i-1} \gets \textsc{OdeStep}(\hat{x}_i, d_i, t, h)$ \Comment{Reverse step: e.g. $\hat{x}_i - h \cdot d_i$}
    \State $t \gets t - h$
\EndFor

\vspace{1.5mm}
\State \Return $\hat{x}_0$
\end{algorithmic}
\end{algorithm}

\section{Complete Proofs for Section 4}
\label{app:theoretical_proofs}

\subsection{Proof of Lemma~\ref{lem:posterior_covariance} (Closed-Form Posterior Covariance)}
\label{app:proof_lemma_posterior}
\begin{proof}
Under Assumption~\ref{ass:gaussian_surrogate}, the conditional prior is $p(x_0 \mid z) \sim \mathcal{N}(\mu_{x \mid z}, \Sigma_\phi)$ with regularized prior precision $\Sigma_\phi^{-1} = \Omega_\phi + \epsilon_0 I_D \succ \mathbf{0}$. 
The linear forward observation model is:
\begin{equation}
    x_t = M_t x_0 + t x_1, \quad x_1 \sim \mathcal{N}(0, I_D).
\end{equation}
Conditioned on $x_0$, the observation distribution is $\mathcal{N}(M_t x_0, t^2 I_D)$. 
The observation likelihood precision is:
\begin{equation}
    \Omega_{\text{lik}} = M_t^T \left( t^2 I_D \right)^{-1} M_t = \frac{1}{t^2} M_t^T M_t.
\end{equation}
By Gaussian conjugacy, posterior precision is additive:
\begin{equation}
    \Sigma_{\text{post}}(t)^{-1} = \Omega_\phi + \epsilon_0 I_D + \frac{1}{t^2} M_t^T M_t = \frac{1}{t^2} \Big( t^2 \Omega_\phi + t^2 \epsilon_0 I_D + M_t^T M_t \Big).
\end{equation}
Substituting the composite drift $M_t = \gamma(t) P_c + \alpha(t) P_s$:
\begin{align}
    M_t^T M_t &= \big(\gamma(t) P_c + \alpha(t) P_s\big)^T \big(\gamma(t) P_c + \alpha(t) P_s\big) \nonumber \\
    &= \gamma(t)^2 P_c^T P_c + \alpha(t)^2 P_s^T P_s + \gamma(t)\alpha(t) (P_c^T P_s + P_s^T P_c).
\end{align}
Applying the orthogonal projector identities ($P_c^T = P_c$, $P_s^T = P_s$, $P_c^2 = P_c$, $P_s^2 = P_s$, and $P_c P_s = \mathbf{0}$) reduces this form to:
\begin{equation}
    M_t^T M_t = \gamma(t)^2 P_c + \alpha(t)^2 P_s.
\end{equation}
Decomposing the identity matrix as $I_D = P_c + P_s$ yields:
\begin{equation}
    t^2 \epsilon_0 I_D + M_t^T M_t = \big(\gamma(t)^2 + t^2 \epsilon_0\big) P_c + \big(\alpha(t)^2 + t^2 \epsilon_0\big) P_s.
\end{equation}
Inverting $\Sigma_{\text{post}}(t)^{-1} = \frac{1}{t^2}\Big(t^2 \Omega_\phi + \big(\gamma(t)^2 + t^2 \epsilon_0\big) P_c + \big(\alpha(t)^2 + t^2 \epsilon_0\big) P_s\Big)$ yields Eq.~\eqref{eq:sigma_post_closed}.
\end{proof}

\subsection{Proof of Theorem~\ref{thm:disentanglement} (Exact Subspace Disentanglement)}
\label{app:proof_thm_disentanglement}
\begin{proof}
Let $V = [V_c, V_s] \in \mathbb{R}^{D \times D}$ be the orthogonal matrix ($V^T V = I_D$) diagonalizing $P_c$ and $P_s$:
\begin{equation}
    V^T P_c V = \begin{pmatrix} I_{D_c} & \mathbf{0} \\ \mathbf{0} & \mathbf{0} \end{pmatrix}, \quad 
    V^T P_s V = \begin{pmatrix} \mathbf{0} & \mathbf{0} \\ \mathbf{0} & I_{D_s} \end{pmatrix}.
\end{equation}
Because the trace is invariant under cyclic orthogonal similarity transformations, the instantaneous loss is:
\begin{equation}
    \mathcal{E}_{t, \,\text{repr}}^*(\phi) = t^2 \operatorname{Tr} \left( V^T \big(\lambda_c(t)^2 P_c + \lambda_s(t)^2 P_s\big) V \cdot \left[ V^T \Sigma_{\text{post}}(t)^{-1} V \right]^{-1} \right).
\end{equation}
Define $\tilde{S}(t) = V^T \big(t^2 \Omega_\phi + (\gamma(t)^2 + t^2 \epsilon_0) P_c + (\alpha(t)^2 + t^2 \epsilon_0) P_s\big) V$. Expanding in block form:
\begin{equation}
    \tilde{S}(t) = \begin{pmatrix} S_{cc}(t) & t^2 \Omega_{cs} \\ t^2 \Omega_{sc} & S_{ss}(t) \end{pmatrix},
\end{equation}
where $S_{cc}(t) = t^2 \Omega_{cc} + (\gamma(t)^2 + t^2 \epsilon_0) I_{D_c} \succ \mathbf{0}$ and $S_{ss}(t) = t^2 \Omega_{ss} + (\alpha(t)^2 + t^2 \epsilon_0) I_{D_s} \succ \mathbf{0}$. 
The transformed weight matrix is block-diagonal: $\operatorname{diag}(\lambda_c(t)^2 I_{D_c}, \lambda_s(t)^2 I_{D_s})$. Multiplying by $\tilde{S}(t)^{-1}$ gives:
\begin{equation}
    \mathcal{E}_{t, \,\text{repr}}^*(\phi) = t^2 \lambda_c(t)^2 \operatorname{Tr} \left( \left[ \tilde{S}(t)^{-1} \right]_{cc} \right) + t^2 \lambda_s(t)^2 \operatorname{Tr} \left( \left[ \tilde{S}(t)^{-1} \right]_{ss} \right).
\end{equation}
By the block inversion formula, the leading diagonal block of $\tilde{S}(t)^{-1}$ is the inverse of the Schur complement of $S_{ss}(t)$:
\begin{equation}
    \left[ \tilde{S}(t)^{-1} \right]_{cc} = \Big( S_{cc}(t) - t^4 \Omega_{cs} S_{ss}(t)^{-1} \Omega_{sc} \Big)^{-1}.
\end{equation}
Let $\Delta_c(t) \triangleq t^4 \Omega_{cs} S_{ss}(t)^{-1} \Omega_{sc}$. Because $S_{ss}(t) \succ \mathbf{0}$, its inverse is positive definite, which guarantees $\Delta_c(t) \succeq \mathbf{0}$ for any coupling matrix $\Omega_{cs}$. 
Under the Loewner partial order on positive-definite matrices:
\begin{equation}
    S_{cc}(t) - \Delta_c(t) \preceq S_{cc}(t) \implies \big( S_{cc}(t) - \Delta_c(t) \big)^{-1} \succeq S_{cc}(t)^{-1}.
\end{equation}
Applying the trace operator yields:
\begin{equation}
    \operatorname{Tr} \left( \big( S_{cc}(t) - \Delta_c(t) \big)^{-1} \right) \ge \operatorname{Tr} \left( S_{cc}(t)^{-1} \right),
\end{equation}
with equality holding if and only if $\Delta_c(t) = \mathbf{0}$. Because $S_{ss}(t)^{-1}$ is full rank, $\Delta_c(t) = \mathbf{0}$ implies $\Omega_{cs} = \mathbf{0}$.

Applying the same reasoning to the trailing diagonal block gives:
\begin{equation}
    \left[ \tilde{S}(t)^{-1} \right]_{ss} = \Big( S_{ss}(t) - t^4 \Omega_{sc} S_{cc}(t)^{-1} \Omega_{cs} \Big)^{-1} \succeq S_{ss}(t)^{-1},
\end{equation}
with equality if and only if $\Omega_{sc} = \Omega_{cs}^T = \mathbf{0}$. 
Because both trace terms are minimized when cross-coupling vanishes, the global optimizer satisfies $\Omega_{cs} = \mathbf{0}$, which proves $[\Omega_\phi^*, P_c] = \mathbf{0}$.
\end{proof}

\subsection{Proof of Lemma~\ref{lem:saturation} (Capacity Saturation)}
\label{app:proof_lemma_saturation}
\begin{proof}
By Theorem~\ref{thm:disentanglement}, $\Omega_\phi^*$ is block-diagonal in basis $V$. It decomposes spectrally as $\Omega_\phi^* = \sum_{j=1}^D \omega_j v_j v_j^T$, where each unit eigenvector $v_j$ lies strictly in $\mathcal{X}_c$ ($\delta_j = 1$) or strictly in $\mathcal{X}_s$ ($\delta_j = 0$). 
The encoder loss separates into a sum over scalar eigenvalues:
\begin{equation}
    \mathcal{E}_{\text{repr}}^*(\phi) = \sum_{j=1}^D \int_0^1 w(t) t^2 \left( \frac{\lambda_c(t)^2 \delta_j}{t^2 \omega_j + \gamma(t)^2 + t^2 \epsilon_0} + \frac{\lambda_s(t)^2 (1 - \delta_j)}{t^2 \omega_j + \alpha(t)^2 + t^2 \epsilon_0} \right) dt.
\end{equation}
Differentiating with respect to an arbitrary eigenvalue $\omega_j$:
\begin{equation}
    \frac{\partial \mathcal{E}_{\text{repr}}^*}{\partial \omega_j} = - \int_0^1 w(t) t^4 \left( \frac{\lambda_c(t)^2 \delta_j}{\big(t^2 \omega_j + \gamma(t)^2 + t^2 \epsilon_0\big)^2} + \frac{\lambda_s(t)^2 (1 - \delta_j)}{\big(t^2 \omega_j + \alpha(t)^2 + t^2 \epsilon_0\big)^2} \right) dt.
\end{equation}
For any $t \in (0, 1)$, $w(t) > 0$ and $t^4 > 0$. The integrand is strictly positive for all $\omega_j \ge 0$, establishing that:
\begin{equation}
    \frac{\partial \mathcal{E}_{\text{repr}}^*}{\partial \omega_j} < 0, \quad \forall \omega_j \in [0, \omega_{\max}].
\end{equation}
Because the objective is strictly monotonically decreasing, minimizing $\mathcal{E}_{\text{repr}}^*$ under the capacity constraints $\sum_{j=1}^D \mathbb{I}(\omega_j > 0) \le c$ and $\omega_j \in [0, \omega_{\max}]$ requires placing all capacity on the upper boundary: $c$ eigenvalues attain $\omega_j^* = \omega_{\max}$, while the remaining $D - c$ eigenvalues equal zero.
\end{proof}

\subsection{Proof of Theorem~\ref{thm:truncation} (Subspace Capacity Allocation via Truncation)}
\label{app:proof_thm_truncation}
\begin{proof}
By Lemma~\ref{lem:saturation}, each active dimension reduces variance by $\Delta C = C(0) - C(\omega_{\max})$. For content ($\delta_j = 1$), this reduction is:
\begin{align}
    \Delta C_c &= \int_0^1 w(t) t^2 \lambda_c(t)^2 \left( \frac{1}{\gamma(t)^2 + t^2 \epsilon_0} - \frac{1}{t^2 \omega_{\max} + \gamma(t)^2 + t^2 \epsilon_0} \right) dt \nonumber \\
    &= \int_0^1 w(t) \frac{\lambda_c(t)^2 t^4 \omega_{\max}}{[\gamma(t)^2 + t^2 \epsilon_0] [t^2 (\omega_{\max} + \epsilon_0) + \gamma(t)^2]} dt.
\end{align}
For style ($\delta_j = 0$), the reduction evaluates to:
\begin{equation}
    \Delta C_s = \int_0^1 w(t) \frac{\lambda_s(t)^2 t^4 \omega_{\max}}{[\alpha(t)^2 + t^2 \epsilon_0] [t^2 (\omega_{\max} + \epsilon_0) + \alpha(t)^2]} dt.
\end{equation}

\noindent\textbf{Part 1: Asymptotic Style Domination under Untruncated Decay ($t_1 = 1$).} 
Substitute $\gamma(t) = 1 - t$ and $\lambda_c(t)^2 = 1/t^2$:
\begin{equation}
    \Delta C_c = \int_0^1 w(t) \frac{t^2 \omega_{\max}}{[(1 - t)^2 + t^2 \epsilon_0] [t^2 (\omega_{\max} + \epsilon_0) + (1 - t)^2]} dt.
\end{equation}
For $t_1 = 1$, the style schedule is $\alpha(t) = (1 - t) e^{-\kappa t}$, and $\lambda_s(t)^2 = \frac{e^{-2\kappa t}}{t^2}[1 + \kappa t(1-t)]^2$. 
The style variance reduction becomes:
\begin{equation}
    \Delta C_s = \int_0^1 w(t) \frac{[1 + \kappa t(1-t)]^2 t^2 \omega_{\max} e^{-2\kappa t}}{[(1 - t)^2 e^{-2\kappa t} + t^2 \epsilon_0] [t^2 (\omega_{\max} + \epsilon_0) + (1 - t)^2 e^{-2\kappa t}]} dt.
\end{equation}
Dividing numerator and the first bracketed denominator term by $e^{-2\kappa t}$:
\begin{equation}
    \Delta C_s = \int_0^1 w(t) \frac{[1 + \kappa t(1-t)]^2 t^2 \omega_{\max}}{[(1 - t)^2 + t^2 \epsilon_0 e^{2\kappa t}] [t^2 (\omega_{\max} + \epsilon_0) + (1 - t)^2 e^{-2\kappa t}]} dt.
\end{equation}
For small regularizations $\epsilon_0 \ll 1$, comparing the integrands pointwise for any $t \in (0, 1)$:
\begin{enumerate}[label=(\roman*),leftmargin=*]
    \item {Numerator:} For any $\kappa > 0$ and $t \in (0, 1)$, $[1 + \kappa t(1-t)]^2 > 1$.
    \item {Denominator:} Because $e^{-2\kappa t} < 1$, the product of the bracketed denominator terms is strictly smaller for the style integrand than for the content integrand.
\end{enumerate}
Because the style integrand dominates the content integrand pointwise across $t \in (0, 1)$, $\Delta C_s > \Delta C_c$. 
In the linear program $\max_{\{\delta_j\}} \sum_{j=1}^c [\delta_j (\Delta C_c - \Delta C_s) + \Delta C_s]$, the coefficient $(\Delta C_c - \Delta C_s)$ is negative, forcing $\delta_j^* = 0$ for all $j \in \{1, \dots, c\}$. 
The untruncated encoder therefore dedicates its bottleneck capacity to the style subspace $\mathcal{X}_s$.

\medskip
\noindent\textbf{Part 2: Content Alignment via Style Truncation ($t_1 < 1$).}
When $t_1 < 1$, the style operator vanishes on $[t_1, 1]$ ($\alpha(t) \equiv 0$ and $\lambda_s(t) \equiv 0$). The style integral truncates to:
\begin{equation}
    \Delta C_s = \int_0^{t_1} w(t) \frac{[1 + \kappa t(1-t)]^2 t^2 \omega_{\max}}{[(1 - t)^2 + t^2 \epsilon_0 e^{2\kappa t}] [t^2 (\omega_{\max} + \epsilon_0) + (1 - t)^2 e^{-2\kappa t}]} dt.
\end{equation}
Because the integration interval $[0, t_1]$ is bounded strictly away from $t = 1$, the denominator satisfies $(1 - t)^2 \ge (1 - t_1)^2 > 0$. 
The integrand is continuous on the compact domain $[0, t_1]$. By the extreme value theorem, there exists a constant $M < \infty$ bounding the integrand, which implies:
\begin{equation}
    \Delta C_s \le M \cdot t_1 < \infty.
\end{equation}
Conversely, content is integrated across the full unit interval $[0, 1]$. Near the boundary $t \to 1$:
\begin{equation}
    \lim_{t \to 1} \frac{t^2 \omega_{\max}}{t^2 (\omega_{\max} + \epsilon_0) + (1 - t)^2} = \frac{\omega_{\max}}{\omega_{\max} + \epsilon_0} > 0.
\end{equation}
Setting $u = 1 - t$, the singular contribution near $t = 1$ is governed by:
\begin{equation}
    \int_0^\delta \frac{du}{u^2 + \epsilon_0} = \frac{1}{\sqrt{\epsilon_0}} \arctan\left(\frac{\delta}{\sqrt{\epsilon_0}}\right) \xrightarrow{\epsilon_0 \to 0^+} \frac{\pi}{2\sqrt{\epsilon_0}} = \mathcal{O}\left(\frac{1}{\sqrt{\epsilon_0}}\right),
\end{equation}
for any small neighborhood $\delta > 0$. 
While $\Delta C_s \le M \cdot t_1 < \infty$ remains uniformly bounded due to truncation at $t_1 < 1$, the content reduction diverges as $\Delta C_c \sim \mathcal{O}(1/\sqrt{\epsilon_0}) \to \infty$ as $\epsilon_0 \to 0^+$ (or $\mathcal{O}(1/\delta_{\text{clamp}})$ under boundary clamping $t \le 1 - \delta_{\text{clamp}}$). 
Consequently, $\Delta C_c \gg \Delta C_s$, yielding $\Delta C_c - \Delta C_s > 0$. The linear code attains its maximum at $\delta_j^* = 1$ for all $j \in \{1, \dots, c\}$, compelling the optimal encoder to allocate its bottleneck capacity exclusively to the content subspace $\mathcal{X}_c$.
\end{proof}

\subsection{Hyperparameter Details and Training Configurations}
\label{app:hyperparameters}

Table~\ref{tab:csfm_params} details the default curriculum schedules, operator-specific exponential decay rates, and optimization parameters used across our experiments.

\begin{table}[htpb]
  \centering
  \caption{CSFM degradation curriculum and optimization hyperparameters.}
  \label{tab:csfm_params}
  \renewcommand{\arraystretch}{1.1}
  \begin{tabular}{@{}lc@{}}
    \toprule
    \textbf{Parameter} & \textbf{Value} \\
    \midrule
    \multicolumn{2}{@{}l}{\textit{Operator Decay Rates}} \\
    \quad Illumination ($k_{\text{illum}}$) & $3.250$ \\ 
    \quad Color ($k_{\text{color}}$) & $2.925$ \\ 
    \quad Frequency ($k_{\text{freq}}$) & $2.632$ \\
    \midrule 
    \multicolumn{2}{@{}l}{\textit{Curriculum Horizons and Projectors}} \\
    \quad Style Truncation Boundary ($t_1$) & $0.50$ \\
    \quad Routing Cutoff ($t_0$) & $0.05$ \\
    \quad Radial Gaussian Cutoff ($r_{\text{low}}$) & $6$ \\
    \midrule
    \multicolumn{2}{@{}l}{\textit{Optimization}} \\
    \quad Optimizer & Adam (Unet) / AdamW (DiT/ViT) \\
    \quad Learning Rate ($\eta$) & $2 \times 10^{-4}$ \\
    \quad Regularization Weight ($\lambda_{\text{reg}}$) & $10^{-3}$ \\
    \quad Mixed Precision & \texttt{bfloat16} \\
    \bottomrule
  \end{tabular}
\end{table}

\section{Baseline Objectives and Discriminative Setup}
\label{app:baseline_details}

All baseline models optimize a composite objective balancing their task-specific loss with the structural regularizer:
\begin{equation}
    \mathcal{L}_{\text{total}} = \mathcal{L}_{\text{task}} + \lambda_{\text{reg}} \mathcal{L}_{\text{SIGReg}}(z).
\end{equation}
Table~\ref{tab:baseline_objectives} defines the task objectives $\mathcal{L}_{\text{task}}$ across the four baseline paradigms. 
Generative and autoencoding baselines use $\lambda_{\text{reg}} = 10^{-2}$. 
For AE, spatial skip connections are zeroed out to force reconstruction solely through the compressed latent bottleneck $z = E_\phi(x_0)$. 
For CDAE, intermediate noisy states $x_t$ are drawn from the continuous interpolation path, with the network trained to predict $x_0$ in a single step. 
OT-Flow implements standard linear flow matching with $x_t = (1-t)x_0 + t x_1$ and constant target velocity $u_t = x_1 - x_0$.

\begin{table}[htpb]
\centering
\caption{Primary task objectives ($\mathcal{L}_{\text{task}}$) for comparative baselines. We denote $x_1 \sim \mathcal{N}(0, I_D)$ as the spatial prior and $v_1, v_2$ as independent stochastic views.}
\label{tab:baseline_objectives}
\renewcommand{\arraystretch}{1.3}
\begin{tabular}{@{}llc@{}}
\toprule
\textbf{Model} & \textbf{Learning Paradigm} & \textbf{Primary Task Objective ($\mathcal{L}_{\text{task}}$)} \\
\midrule
\textbf{AE} & Unconstrained Reconstruction & $\mathbb{E}_{x_0} \left[ \big\| x_0 - v_\theta(\mathbf{0}, 0, z) \big\|_2^2 \right]$ \\
\textbf{CDAE} & Single-Step Denoising & $\mathbb{E}_{t, x_0, x_1} \left[ \big\| x_0 - \hat{x}_\theta(x_t, t, z) \big\|_2^2 \right]$ \\
\textbf{OT-Flow} & Continuous Linear Flow & $\mathbb{E}_{t, x_0, x_1} \left[ \big\| v_\theta(x_t, t, z) - (x_1 - x_0) \big\|_2^2 \right]$ \\
\textbf{LeJEPA} &  Joint Embedding & $\mathbb{E}_{x_0, v_1, v_2} \left[ \frac{1}{2} \big\| z_{v_1} - z_{v_2} \big\|_2^2 \right]$ \\
\bottomrule
\end{tabular}
\end{table}

\paragraph{LeJEPA Projector Architecture.}
To avoid dimensional collapse, the LeJEPA baseline appends a 3-layer MLP projector to the WideResNet backbone following VICReg \cite{bardes2022vicregvarianceinvariancecovarianceregularizationselfsupervised}. 
The projector expands the backbone feature dimension ($d=512$) to a 2048-dimensional hidden layer before bottlenecking to the final representation $z \in \mathbb{R}^{256}$. 
The first two linear layers are followed by 1D batch normalization and ReLU activations to bound variance, while the final layer is linear.

\paragraph{Dual-View Data Augmentations.}
LeJEPA optimizes invariance over asymmetric view pairs $(v_1, v_2)$ constructed according to BYOL \cite{grill2020bootstraplatentnewapproach} and DINO \cite{caron2021emergingpropertiesselfsupervisedvision} specifications adapted for CIFAR-10 ($32 \times 32$):
\begin{itemize}[leftmargin=*]
    \item \textbf{Common Transformations:} Both views share random resized cropping (scale $[0.2, 1.0]$, bicubic interpolation), random horizontal flipping ($p=0.5$), color jittering ($p=0.8$, with brightness, contrast, and saturation factors of $0.4$, and hue factor of $0.1$), and random grayscale conversion ($p=0.2$).
    \item \textbf{Asymmetric Blurring:} View 1 applies Gaussian blur with probability $p=0.999$, using a $3 \times 3$ kernel ($\approx 10\%$ of spatial resolution) and $\sigma \in [0.1, 2.0]$. View 2 applies the identical blur kernel with probability $p=0.1$.
\end{itemize}

\paragraph{LeJEPA Regularization Strength ($\lambda_{\text{reg}}$) Ablation.}
Because batch normalization explicitly bounds embedding variance, the balance between invariance alignment and the SIGReg penalty is sensitive. 
Table~\ref{tab:jea_ablation} benchmarks the effect of varying $\lambda_{\text{reg}}$ for LeJEPA. 
Setting $\lambda_{\text{reg}} \le 1 \times 10^{-3}$ yields insufficient repulsive force, causing collapse and low linear accuracy ($18.62\%$). 
Conversely, scaling $\lambda_{\text{reg}} > 3 \times 10^{-2}$ prioritizes uniform feature spreading over semantic view clustering. 
The optimum occurs at $\lambda_{\text{reg}} = 3 \times 10^{-2}$, which is adopted as the default across all LeJEPA benchmarks.

\begin{table}[htpb]
    \centering
    \caption{Ablation of SIGReg regularization strength ($\lambda_{\text{reg}}$) on the LeJEPA baseline.}
    \label{tab:jea_ablation}
    \renewcommand{\arraystretch}{1.15}
    \begin{tabular}{@{}lccccc@{}}
        \toprule
        \multirow{2}{*}{\textbf{Regularization ($\lambda_{\text{reg}}$)}} & \multicolumn{3}{c}{\textbf{Linear Probing Accuracy}} & \multirow{2}{*}{\textbf{$k$-NN ($k=20$)}} & \multirow{2}{*}{\textbf{Silhouette Score}} \\
        \cmidrule(lr){2-4}
        & \textbf{100\% Data} & \textbf{10\% Data} & \textbf{1\% Data} & & \\
        \midrule
        $1 \times 10^{-3}$ & 18.62\% & 18.12\% & 17.85\% & 31.13\% & -0.1763 \\
        $5 \times 10^{-3}$ & 65.18\% & 59.81\% & 54.57\% & 64.41\% & -0.0535 \\
        $7.5 \times 10^{-3}$ & 71.31\% & 68.28\% & 64.21\% & 70.31\% & 0.1019 \\
        $1 \times 10^{-2}$ & \textbf{72.97\%} & 70.26\% & 66.69\% & 71.64\% & 0.1604 \\
        $3 \times 10^{-2}$ & 72.49\% & \textbf{71.02\%} & \textbf{67.76\%} & \textbf{71.65\%} & \textbf{0.1738} \\
        $5 \times 10^{-2}$ & 71.15\% & 69.79\% & 66.62\% & 70.41\% & 0.1372 \\
        \bottomrule
    \end{tabular}
\end{table}

\section{Extended CIFAR-10 Empirical Results and Generative Validation}
\label{app:extended_experiments}

This section provides extended empirical evaluations on CIFAR-10, detailing generative synthesis metrics, numerical ODE trajectory curvature, fine-grained ablations on physical decay parameters and latent regularization weights, and manifold continuity analysis.

\subsection{Latent Space Structural Geometry: SVD Spectrum and Feature Covariance}
\label{app:latent_structure}

To evaluate how different learning objectives shape the latent bottleneck, we analyze the singular value decomposition (SVD) spectrum, empirical Pearson feature correlations, and topological clustering on frozen CIFAR-10 test embeddings ($z = E_\phi(x_0)$) across OT-Flow, truncated CSFM ($t_1=0.5$), and LeJEPA.

\paragraph{Capacity Utilization and Effective Dimensionality.}
Figure~\ref{fig:svd_analysis} reports the normalized explained variance spectrum and cumulative energy thresholds across baselines. 
We quantify active capacity via the effective dimensionality $d_{90\%}$, the minimum number of principal components capturing 90\% of embedding variance. 
LeJEPA exhibits severe dimensional collapse ($d_{90\%} = 9$), collapsing features onto the minimal simplex required to discriminate the 10 training classes \cite{jing2022understandingdimensionalcollapsecontrastive}. 
While standard OT-Flow mitigates collapse, its spectrum saturates at ($d_{90\%} = 31$). 
In contrast, truncated CSFM distributes semantic features across a wider spectrum ($d_{90\%} = 37$), preserving the high-rank representation capacity required for reverse-ODE synthesis without sacrificing class separability.

\begin{figure}[htpb]
    \centering
    \includegraphics[width=0.85\linewidth]{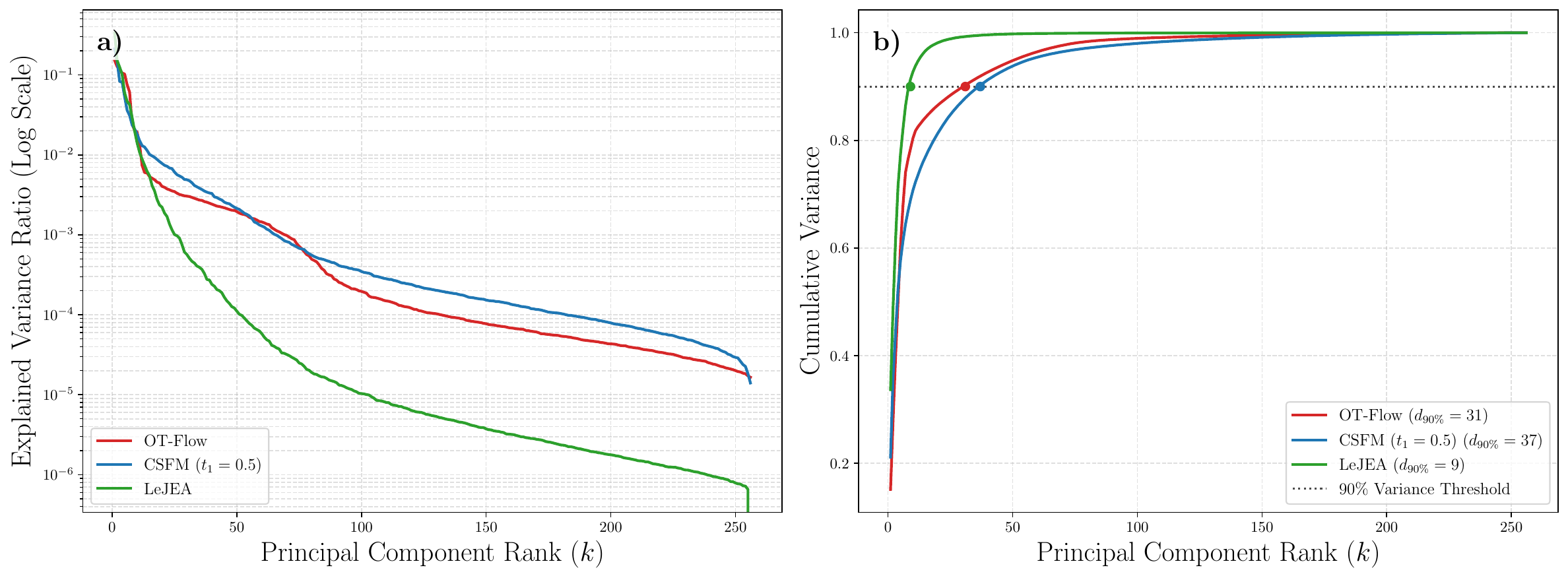}
    \caption{Dimensional collapse analysis via singular value decomposition (SVD) on frozen CIFAR-10 test embeddings. (\textbf{a}) Normalized singular value spectrum. (\textbf{b}) Cumulative explained variance highlighting the 90\% energy threshold ($d_{90\%}$).}
    \label{fig:svd_analysis}
\end{figure}

\paragraph{Feature Cross-Covariance Structure.}
Figure~\ref{fig:correlation_analysis} visualizes the pairwise Pearson correlation matrix across all latent dimensions. 
Under standard OT-Flow, off-diagonal entries are uniformly near zero, indicating representational uncoupling and under-utilized capacity \cite{lucas2019dontblameelbolinear}. 
LeJEPA exhibits dense, block-structured correlations induced by augmentations forcing co-occurring semantic features together. 
CSFM ($t_1=0.5$) interpolates between these two regimes: it maintains mild structured off-diagonal couplings that reflect shared macroscopic attributes.
\begin{figure}[htpb]
    \centering
    \includegraphics[width=\linewidth]{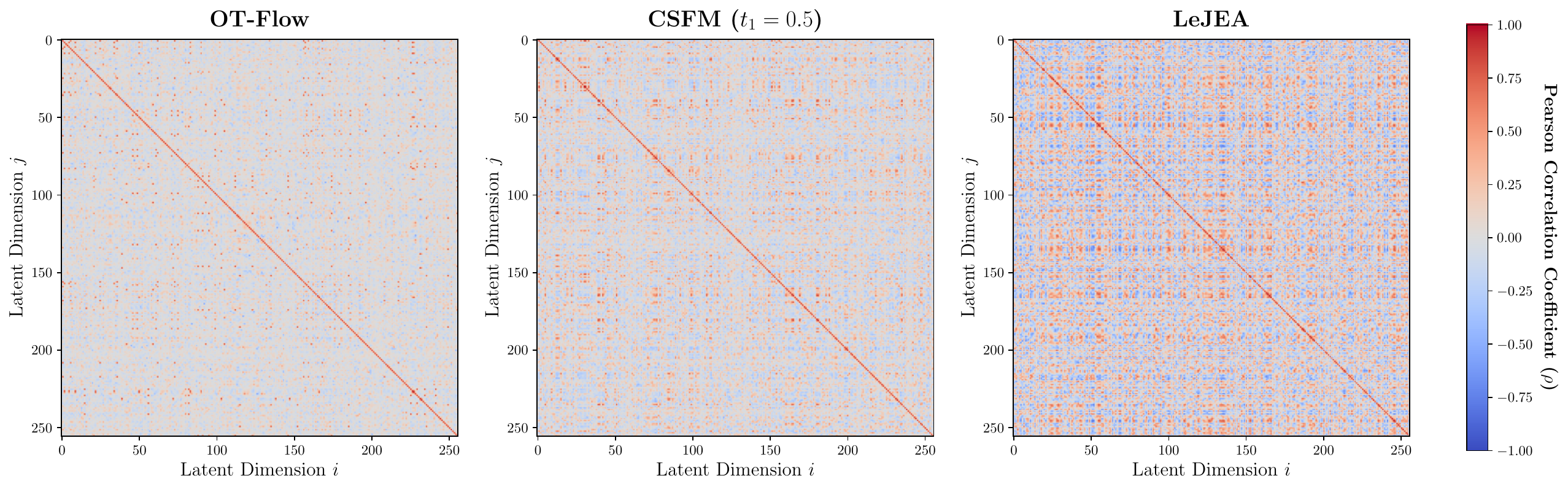}
    \caption{Empirical Pearson correlation matrices of the latent bottleneck dimensions on CIFAR-10 test embeddings across OT-Flow (\textbf{Left}), truncated CSFM ($t_1=0.5$, \textbf{Middle}), and LeJEPA (\textbf{Right}).}
    \label{fig:correlation_analysis}
\end{figure}

\paragraph{Manifold Connectivity and Topological Dead Zones.}
As illustrated by the $t$-SNE projections in Figure~\ref{fig:tsne_analysis}, contrastive margin maximization shatters the LeJEPA embedding space into isolated class clusters separated by low-density \emph{dead zones} \cite{wang2022understandingcontrastiverepresentationlearning}. 
While advantageous for $k$-NN and few-shot classification, reverse-ODE paths passing through these empty regions yield severe out-of-distribution synthesis artifacts \cite{white2016samplinggenerativenetworks}. 
Conversely, unguided OT-Flow maintains dense continuity but suffers from diffuse, overlapping boundaries. 
Truncated CSFM preserves global manifold connectivity, enabling both reliable linear readout and stable bidirectional transport.

\subsection{Sensitivity to Latent Regularization ($\lambda_{\text{reg}}$)}
\label{app:ablation_sigreg}

Table~\ref{tab:app_ablation_reg} reports representation metrics across regularization weights $\lambda_{\text{reg}}$ on the 20.81M parameter architecture (40K optimization steps). 
Omitting the regularizer ($\lambda_{\text{reg}} = 0$) preserves high downstream linear accuracy (68.98\%), confirming that CSFM avoids representational collapse on its own. 
The optimal trade-off occurs at $\lambda_{\text{reg}} = 10^{-3}$, which peaks at 69.09\% linear accuracy and 46.41\% 1\% few-shot accuracy, while excessive regularization ($\lambda_{\text{reg}} \ge 0.1$) degrades downstream separability.

\begin{table}[htpb]
\centering
\caption{Ablation over structural regularization weight $\lambda_{\text{reg}}$ on the 20.81M baseline tier (40K steps). CSFM evaluated with $\kappa = 6.0, \beta = 0.7, t_0 = 0.0, t_1 = 1.0$.}
\label{tab:app_ablation_reg}
\renewcommand{\arraystretch}{1.1}
\begin{tabular}{@{}llcccc@{}}
\toprule
\textbf{Model} & \textbf{$\lambda_{\text{reg}}$} & \textbf{Linear Probe} & \textbf{1\% Few-Shot} & \textbf{$k$-NN ($k=20$)} & \textbf{Silhouette} \\
\midrule
\multirow{5}{*}{\textbf{CSFM}} 
& $0.0$ (Unregularized) & 68.98\% & 45.50\% & 64.08\% & -0.0077 \\
& $\mathbf{1 \times 10^{-3}}$ & \textbf{69.09\%} & \textbf{46.41\%} & 63.13\% & -0.0185 \\
& $5 \times 10^{-3}$ & 68.09\% & 46.40\% & 53.85\% & -0.0485 \\
& $1 \times 10^{-2}$ & 68.44\% & 46.31\% & 53.20\% & -0.0221 \\
& $1 \times 10^{-1}$ & 64.61\% & 44.96\% & 55.73\% & \textbf{0.0001} \\
\midrule
\multirow{4}{*}{OT-Flow} 
& $0.0$ (Unregularized) & 66.40\% & 43.61\% & 60.40\% & -0.0205 \\
& $1 \times 10^{-3}$ & 66.93\% & 42.72\% & \textbf{61.05\%} & -0.0190 \\
& $\mathbf{1 \times 10^{-2}}$ & \textbf{67.11\%} & \textbf{45.75\%} & 52.97\% & -0.0454 \\
& $1 \times 10^{-1}$ & 65.07\% & 45.24\% & 56.35\% & \textbf{-0.0014} \\
\bottomrule
\end{tabular}
\end{table}

\subsection{Physical Degradation Parameter Sensitivity}
\label{app:physics_ablation}

Table~\ref{tab:app_ablation_physics} summarizes sensitivity sweeps over the base style decay rate $\kappa$ and operator decay factor $\beta$, where per-subspace rates follow the geometric schedule $\kappa_i = \kappa \beta^i$ indexed across the physical degradation hierarchy ($i \in \{0, 1, 2\}$ for illumination, color, and frequency). 
Under optimal truncation ($t_1=0.5$), downstream performance remains robust across five orders of magnitude ($\kappa = 10^{-4}$ to $6.0$), with optimal linear probing reaching 76.10\% ($\kappa = 3.25, \beta = 0.7$). Performance decreases only at extreme decay stiffness ($\kappa \ge 8.75$). 
Conversely, without truncation ($t_1=1.0$), varying $\kappa$ and $\beta$ yields flat downstream performance bounded below 69.61\%, confirming that continuous decay schedules cannot recover content features without boundary truncation.

\begin{table}[htpb]
\centering
\caption{Ablation of continuous trajectory decay parameters ($\kappa$ and $\beta$) on CIFAR-10 (20.81M tier, 40K steps).}
\label{tab:app_ablation_physics}
\renewcommand{\arraystretch}{1.15}
\begin{tabular}{@{}lcccc@{}}
\toprule
\textbf{Hyperparameter} & \textbf{Linear Probe} & \textbf{1\% Few-Shot} & \textbf{$k$-NN} & \textbf{Silhouette} \\
\midrule
\multicolumn{5}{c}{\textbf{Optimal Truncation ($t_1 = 0.5$)}} \\
\midrule
\multicolumn{5}{l}{\textit{Base Style Decay Rate ($\kappa$) with $\beta = 0.7$}} \\
$\kappa = 10^{-4}$ & 75.53\% & \textbf{60.12\%} & 66.33\% & \textbf{0.0062} \\
$\kappa = 0.01$ & 74.80\% & 58.46\% & 64.47\% & 0.0011 \\
$\kappa = 0.5$ & 75.25\% & 58.72\% & 65.60\% & -0.0047 \\
\textbf{$\kappa = 3.25$} & \textbf{76.10\%} & 58.57\% & 67.27\% & -0.0169 \\
$\kappa = 6.0$ & 75.12\% & 57.40\% & \textbf{68.35\%} & -0.0215 \\
$\kappa = 8.75$ & 72.43\% & 54.64\% & 65.80\% & -0.0279 \\
$\kappa = 11.5$ & 70.26\% & 51.56\% & 63.09\% & -0.0300 \\
\addlinespace
\multicolumn{5}{l}{\textit{Operator Decay Scaling Factor ($\beta$) with $\kappa = 3.25$}} \\
$\beta = 0.2$ & 72.37\% & 55.57\% & 62.82\% & -0.0214 \\
\textbf{$\beta = 0.7$} & \textbf{76.10\%} & 58.57\% & 67.27\% & -0.0169 \\
$\beta = 0.9$ & 75.29\% & \textbf{58.93\%} & \textbf{68.34\%} & \textbf{-0.0154} \\
\midrule
\multicolumn{5}{c}{\textbf{No Truncation ($t_1 = 1.0$)}} \\
\midrule
\multicolumn{5}{l}{\textit{Base Style Decay Rate ($\kappa$) with $\beta = 0.7$}} \\
$\kappa = 10^{-4}$ & 67.43\% & 42.59\% & 60.77\% & -0.0220 \\
$\kappa = 0.01$ & 67.52\% & 42.71\% & 61.41\% & -0.0211 \\
$\kappa = 0.5$ & 67.93\% & 43.14\% & 61.04\% & -0.0176 \\
$\kappa = 3.25$ & 69.59\% & 46.31\% & 64.14\% & -0.0098 \\
$\kappa = 6.0$ & 69.52\% & 47.41\% & \textbf{64.38\%} & -0.0105 \\
$\kappa = 8.75$ & 69.30\% & 47.81\% & 63.86\% & -0.0209 \\
\textbf{$\kappa = 11.5$} & \textbf{69.60\%} & \textbf{49.60\%} & 62.81\% & -0.0280 \\
\addlinespace
\multicolumn{5}{l}{\textit{Operator Decay Scaling Factor ($\beta$) with $\kappa = 3.25$}} \\
$\beta = 0.2$ & 67.26\% & 43.05\% & 61.05\% & -0.0203 \\
$\beta = 0.7$ & 69.59\% & \textbf{46.31\%} & 64.14\% & -0.0098 \\
\textbf{$\beta = 0.9$} & \textbf{69.61\%} & 46.28\% & \textbf{64.16\%} & \textbf{-0.0127} \\
\bottomrule
\end{tabular}
\end{table}

\subsection{Generative Synthesis Fidelity and ODE Trajectory Curvature}
\label{app:generative_eval}

We evaluate the generative capabilities of the converged vector field regressors ($v_\theta$, 59.10M parameter tier at 120K steps) in Table~\ref{tab:app_fid}. 
We report relative Fr\'echet inception distance ($\text{FID}_{1\text{k}}$) over 1000 synthesized samples against 1000 real test images using the Heun second-order solver (200 function evaluations) for both unconditional sampling ($z \sim \mathcal{N}(0, I_c)$) and conditional reconstruction ($z = E_\phi(x_0)$). 
Additionally, we evaluate conditional ODE truncation error (mean squared error against a ground-truth $N=1000$ Euler baseline) for budget-constrained step counts $N \in \{10, 50\}$.

\begin{table}[htpb]
\centering
\caption{Generative synthesis fidelity ($\text{FID}_{1\text{k}}$) and conditional ODE integration error (MSE vs.\ $N=1000$ baseline) at 120K steps on the high-capacity tier (65.08M total). Best results in {bold}.}
\label{tab:app_fid}
\renewcommand{\arraystretch}{1.12}
\begin{tabular}{@{}lcccc@{}}
\toprule
\textbf{Model Configuration} & \textbf{Uncond. $\text{FID}_{1\text{k}}$ ($\downarrow$)} & \textbf{Cond. $\text{FID}_{1\text{k}}$ ($\downarrow$)} & \textbf{ODE Error ($N=50$)} & \textbf{ODE Error ($N=10$)} \\
\midrule
OT-Flow & \textbf{93.57} & 119.04 & \textbf{0.00015} & \textbf{0.00327} \\
CSFM ($t_1=1.0$) & 125.97 & \textbf{48.84} & 0.00505 & 0.08340 \\
\textbf{CSFM ($t_1=0.5$)} & 95.76 & 87.88 & 0.00277 & 0.04393 \\
\bottomrule
\end{tabular}
\end{table}

\paragraph{Analysis of Generative Trade-offs.}
Because standard OT-Flow optimizes straight paths without a content bottleneck, it achieves the lowest unconditional prior synthesis ($\text{FID}_{1\text{k}}$ 93.57) and minimal numerical integration error. 
However, it fails at conditional reconstruction ($\text{FID}_{1\text{k}}$ 119.04), as unguided flow paths cannot resolve the ambiguity of conditioning codes. 
Conversely, un-truncated CSFM ($t_1=1.0$) memorizes fine textures, minimizing conditional FID (48.84) but inducing severe trajectory stiffness that degrades unconditional prior sampling ($\text{FID}_{1\text{k}}$ 125.97) and inflates numerical error ($0.08340$ at $N=10$). 
Truncating style decay at $t_1=0.5$ addresses this trade-off: it nearly halves numerical integration error ($0.04393$) with conditional (87.88) and unconditional (95.76) synthesis.

\subsection{Latent Manifold Smoothness, Interpolation, and Extrapolation}
\label{app:latent_slerp}

Because the latent bottleneck operates in dimension $b = 128$, the Gaussian prior regularized by SIGReg concentrates mass on a spherical shell of radius $\sqrt{b}$ \cite{vershynin2018high}. 
Standard Euclidean linear interpolation traverses through the unpopulated interior of this shell, inducing low-density artifacts during reverse ODE integration \cite{white2016samplinggenerativenetworks}. 
We therefore evaluate manifold continuity via spherical linear interpolation (Slerp) \cite{white2016samplinggenerativenetworks}, setting $\lambda \in [0, 1]$ for interpolation between test-set anchor pairs $z_A, z_B = E_\phi(x_0)$ and $\lambda \notin [0, 1]$ for extrapolation.

Figure~\ref{fig:slerp_extrapolation_app} compares traversal trajectories across OT-Flow, un-truncated CSFM ($t_1 = 1.0$), and truncated CSFM ($t_1 = 0.5$):
\begin{itemize}[leftmargin=*]
    \item \textbf{Interpolation ($\lambda \in [0, 1]$):} Standard OT-Flow exhibits abrupt semantic jumps due to early variance saturation. Un-truncated CSFM ($t_1=1.0$) preserves localized high-frequency textures but produces discontinuous structural morphing. In contrast, truncated CSFM ($t_1=0.5$) yields continuous semantic transitions, smoothly interpolating class geometry and color distributions.
    \item \textbf{Extrapolation ($\lambda \notin [0, 1]$):} Beyond the anchor coordinates, truncated CSFM maintains coherent global structures without severe distortion, confirming that the vector field learns stable, persistent semantic directions.
\end{itemize}

\begin{figure*}[htpb]
    \centering
    \begin{subfigure}[b]{0.31\textwidth}
        \centering
        \begin{tikzpicture}
            \node[anchor=south west, inner sep=0] (img) at (0,0) {
                \includegraphics[width=\linewidth]{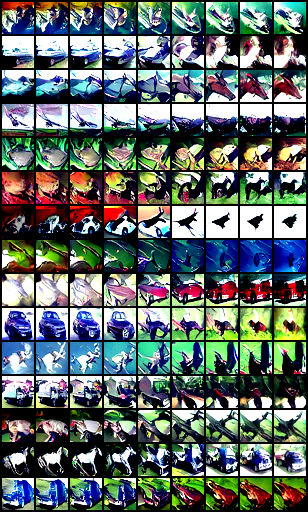}
            };
            \begin{scope}[x={(img.south east)}, y={(img.north west)}]
                \def\ybar{1.06}
                \draw[|-|, thick, red] (0, \ybar) -- (2/9, \ybar) node[midway, above, font=\tiny] {Extrap.};
                \draw[|-|, thick, green!50!black] (2/9, \ybar) -- (7/9, \ybar) node[midway, above, font=\tiny] {Interpolation};
                \draw[|-|, thick, red] (7/9, \ybar) -- (1, \ybar) node[midway, above, font=\tiny] {Extrap.};
            \end{scope}
        \end{tikzpicture}
        \caption{OT-Flow Baseline}
    \end{subfigure}
    \hfill
    \begin{subfigure}[b]{0.31\textwidth}
        \centering
        \begin{tikzpicture}
            \node[anchor=south west, inner sep=0] (img) at (0,0) {
                \includegraphics[width=\linewidth]{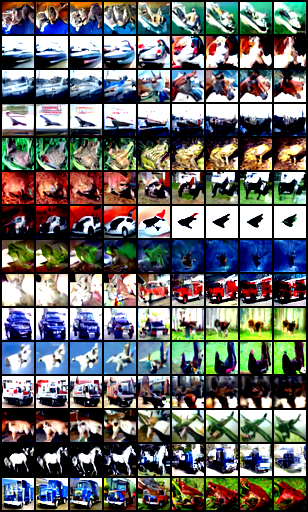}
            };
            \begin{scope}[x={(img.south east)}, y={(img.north west)}]
                \def\ybar{1.06}
                \draw[|-|, thick, red] (0, \ybar) -- (2/9, \ybar) node[midway, above, font=\tiny] {Extrap.};
                \draw[|-|, thick, green!50!black] (2/9, \ybar) -- (7/9, \ybar) node[midway, above, font=\tiny] {Interpolation};
                \draw[|-|, thick, red] (7/9, \ybar) -- (1, \ybar) node[midway, above, font=\tiny] {Extrap.};
            \end{scope}
        \end{tikzpicture}
        \caption{CSFM ($t_1=1.0$, No Truncation)}
    \end{subfigure}
    \hfill
    \begin{subfigure}[b]{0.31\textwidth}
        \centering
        \begin{tikzpicture}
            \node[anchor=south west, inner sep=0] (img) at (0,0) {
                \includegraphics[width=\linewidth]{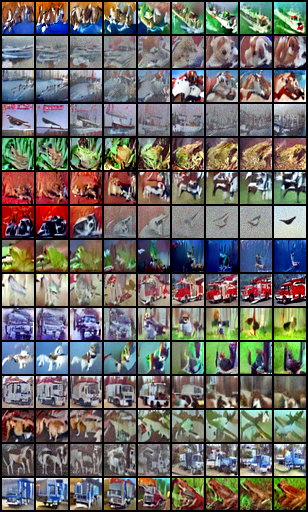}
            };
            \begin{scope}[x={(img.south east)}, y={(img.north west)}]
                \def\ybar{1.06}
                \draw[|-|, thick, red] (0, \ybar) -- (2/9, \ybar) node[midway, above, font=\tiny] {Extrap.};
                \draw[|-|, thick, green!50!black] (2/9, \ybar) -- (7/9, \ybar) node[midway, above, font=\tiny] {Interpolation};
                \draw[|-|, thick, red] (7/9, \ybar) -- (1, \ybar) node[midway, above, font=\tiny] {Extrap.};
            \end{scope}
        \end{tikzpicture}
        \caption{CSFM ($t_1=0.5$, Truncated)}
    \end{subfigure}
    \caption{Spherical linear interpolation ($\lambda \in [0, 1]$, green) and extrapolation ($\lambda \notin [0, 1]$, red) across test-set latent anchors evaluated at 120K steps. Truncated CSFM ($t_1=0.5$) demonstrates fluid semantic morphing across intermediate coordinates and maintains structural integrity into the extrapolation regime.}
    \label{fig:slerp_extrapolation_app}
\end{figure*}

\end{document}